\documentclass{article}

\PassOptionsToPackage{round}{natbib}

\usepackage[final]{neurips_2025}

\usepackage[utf8]{inputenc} 
\usepackage[T1]{fontenc}    
\usepackage[breaklinks=true,
colorlinks=true, 
linkcolor=BrickRed, 
urlcolor=DarkGreen, 
citecolor=C2, 
anchorcolor=DarkGreen]{hyperref}       

\usepackage{booktabs}       
\usepackage{amsfonts}       
\usepackage{bbm}
\usepackage{nicefrac}       
\usepackage{microtype}      
\usepackage[dvipsnames,table]{xcolor}         
\usepackage[most]{tcolorbox}
\usepackage{algorithm2e}
\usepackage{amssymb, amsthm}
\RestyleAlgo{ruled}
\usepackage{caption}
\usepackage{multirow}
\usepackage{minitoc}
\usepackage{nomencl}
\usepackage{wrapfig}
\usepackage[shortlabels]{enumitem}
\usepackage{tikz}
\usetikzlibrary{positioning,calc,arrows.meta,shapes.geometric,fit}
\usepackage{cleveref}
\usepackage{notation}
\usepackage{thmtools}
\usepackage{thm-restate}
\usepackage{subcaption}
\usepackage{authblk}

\definecolor{Green}{HTML}{17891a}
\definecolor{DarkGreen}{HTML}{054802}
\definecolor{C0}{HTML}{1f77b4}
\definecolor{C1}{HTML}{ff7f0e}
\definecolor{C2}{HTML}{55a868}
\definecolor{C3}{HTML}{c44e52}

\newtcolorbox{empheqboxed}{colback=Gray!20, 
 colframe=white,
 width=\textwidth,
 sharpish corners,
 top=1mm, 
 bottom=0pt,
 left=2pt,
 right=2pt
}

\title{The third pillar of causal analysis? A measurement perspective on causal representations}

\makeatletter \renewcommand\AB@affilsepx{\, \,   \protect\Affilfont} 
\renewcommand\AB@affilsepx{\\[0.3ex] \protect\Affilfont}
\makeatother

\renewcommand\Affilfont{\small\normalfont}   

\makeatother

\author[1]{\textbf{Dingling Yao}$^*$}
\author[1]{\textbf{Shimeng Huang}$^*$}
\author[1]{\textbf{Riccardo Cadei}}
\author[2, 3]{\textbf{Kun Zhang}}
\author[1]{\textbf{Francesco Locatello}}
\affil[1]{Institute of Science and Technology Austria}
\affil[2]{Carnegie Mellon University}
\affil[3]{Mohamed bin Zayed University of Artificial Intelligence (MBZUAI)}
\affil[*]{\textit{Equal contribution.}} 

\begin{document}
\doparttoc 
\faketableofcontents 
\maketitle

\begin{abstract}

\looseness=-1 Causal reasoning and discovery, two fundamental tasks of causal analysis, often face challenges in applications due to the complexity, noisiness, and high-dimensionality of real-world data.
Despite recent progress in identifying latent causal structures using causal representation learning (CRL), what makes learned representations useful for causal downstream tasks and how to evaluate them are still not well understood.
In this paper, we reinterpret CRL using a measurement model framework, where the learned representations are viewed as proxy measurements of the latent causal variables. Our approach clarifies the conditions under which learned representations support downstream causal reasoning and provides a principled basis for quantitatively assessing the quality of representations using a new Test-based Measurement EXclusivity (T-MEX) score. 
We validate T-MEX across diverse causal inference scenarios, including numerical simulations and real-world ecological video analysis, demonstrating that the proposed framework and corresponding score effectively assess the identification of learned representations and their usefulness for causal downstream tasks. Reproducible code can be found at \href{https://github.com/shimenghuang/a-measurement-perspective-of-crl}{https://github.com/shimenghuang/a-measurement-perspective-of-crl}.
\end{abstract}

\section{Introduction} \label{sec:intro}
\looseness=-1{Causal analysis rests on two foundational pillars: causal reasoning and causal discovery. 
Causal reasoning operates under the assumption that the causal structure is known or can be assumed, and leverages data to make quantitative causal statements, for example, about the average effect of one variable on another. As causal structures are often unknown, causal discovery aims to uncover this structure, assuming that the causal variables of interest are readily observed. 
In many real-world settings, however, the causal variables may not be directly observable. While originally formulated mostly to enable causal capabilities in machine learning models, Causal Representation Learning \citep[CRL,][]{scholkopf2021toward} has the potential to serve as a third pillar of causal analysis: enabling applications of causality involving unstructured data. For this, we reinterpret causal representation learning using the formalism of ``\textit{measurement models}''~\citep{silva2006learning}, 
wherein the learned representations serve as proxy measurements for latent causal variables. This perspective 
of CRL allows us to better characterize when a representation supports downstream causal reasoning, and it also provides a principled basis for quantitatively evaluating the quality of identification.

\looseness=-1 Methodologically, CRL tackles a more challenging task compared to independent component analysis (ICA) and disentanglement, where the latent variables are assumed to be independent of each other~\citep{hyvarinen1999nonlinear,hyvarinen2019nonlinear,higgins2017beta,locatello2019challenging}. Instead, CRL aims to unmix a set of causally related latent variables. 
Many recent causal representation learning works have provided different theoretical results for causal variable identification compiling various problem settings~\citep{von2021self,von2024nonparametric,zhang2024causal,ahuja2024multi,ahuja2022weakly,varici2024general,zhang2024identifiability,yao2023multi,kong2022partial,lippe2022citris,xie2024generalized,dong2023versatile,lachapelle2022disentanglement,lachapelle2022synergies,yao2022temporally,zhang2024identifiability,squires2023linear,buchholz2023learning,kong2023identification}, recently unified by~\citep{yao2024unifying} into a single general methodology. Although most of the results have been theoretical in nature, machine learning models explicitly empowered with identified causal structure have been shown to be more robust under distributional shifts and provide better out-of-distribution generalization~\citep{fumero2024leveraging, ahuja2021invariance,bareinboim2016causal,zhang2020domain,rojas2018invariant}. 
From an AI for science perspective, CRL has shown its potential in understanding climate physics from raw measurement data~\citep{yao2024marrying}, answering causal questions in the scope of ecology experiments~\citep{cadei2024smoke,cadei2025causal,yao2024unifying}, psychometric studies \citep{dong2023versatile}, and countless more applications related to biomedicine~\citep{zhang2024identifiability,sun2025causal,ravuri2025weakly,jain2024automated}.

\looseness=-1Despite recent progress in identifying latent causal structures within causal representation learning, it remains unclear what makes learned representations useful for downstream causal tasks and how to best evaluate them. Building on the proposed measurement model framework, we introduce a new evaluation metric, the Test-based Measurement EXclusivity (T-MEX) Score, which effectively quantifies how well the learned representation aligns with the underlying measurement model. This underlying measurement model can be specified by, for instance, identifiability theory of a CRL algorithm~\pcref{fig:measurement_models}, assumptions for a particular causal reasoning task~\pcref{fig:simu_measurement_model,fig:istan_measurement_model}, or ground truth knowledge.
In contrast to commonly used CRL evaluation metrics, which suffer from clear limitations~\pcref{sec:problems_of_current_eval}, we demonstrate that T-MEX reliably assesses both the identifiability~\pcref{def:identif} and causal validity~\pcref{def:causally_valid_model} of learned representations, as shown in a wide range of causal reasoning tasks across numerical simulations and real-world ecological video analysis~\pcref{sec:exp}. 
We summarize the main contributions of this paper as follows:
\begin{wrapfigure}{r}{0.46\textwidth}
\vspace{-15pt}
  \begin{center}
  \resizebox{\linewidth}{!}{\begin{tikzpicture}
    \tikzstyle{var}=[circle, draw, thick, minimum size=8.5mm, font=\small, inner sep=1]
    \tikzstyle{arrow}=[-latex, thick]
    \tikzstyle{doublearrow}=[latex-latex, thick]
    \tikzstyle{dashedarrow}=[-latex, thick, dashed]

    \node[var] (Z11) at (-3, 1.5) {$\Zb_1$};
    \node[var] (Z12) at (-1.5, 1.5) {$\Zb_2$};
    \node[var] (Z13) at (0, 1.5) {$\Zb_3$};
    \node[var, fill=gray!30] (X) at (-1.5, 0) {$\Xb$};
    \draw[arrow] (Z11) -- (Z12); 
    \draw[arrow] (Z12) -- (Z13);
    \draw[-latex, thick] (Z11) to [out=45,in=145] (Z13);
    \draw[arrow] (Z11) -- (X); 
    \draw[arrow] (Z12) -- (X);
    \draw[arrow] (Z13) -- (X);

    \node[var] (Z21) at (4.5, 4) {$\Zb_1$};
    \node[var] (Z22) at (6, 4) {$\Zb_2$};
    \node[var] (Z23) at (7.5, 4) {$\Zb_3$};
    \node[var, fill=gray!30] (Zh2A1) at (4.5, 2.5) {$\widehat{\Zb}_{A_1}$};
    \node[var, fill=gray!30] (Zh2A2) at (6, 2.5) {$\widehat{\Zb}_{A_2}$};
    \node[var, fill=gray!30] (Zh2A3) at (7.5, 2.5) {$\widehat{\Zb}_{A_3}$};
    \draw[arrow] (Z21) -- (Z22); 
    \draw[arrow] (Z22) -- (Z23);
    \draw[-latex, thick] (Z21) to [out=45,in=145] (Z23);
    \draw[arrow] (Z21) -- (Zh2A1); 
    \draw[arrow] (Z22) -- (Zh2A2); 
    \draw[arrow] (Z23) -- (Zh2A3); 

    \node[var] (Z31) at (4.5, 0.5) {$\Zb_1$};
    \node[var] (Z32) at (6, 0.5) {$\Zb_2$};
    \node[var] (Z33) at (7.5, 0.5) {$\Zb_3$};
    \node[var, fill=gray!30] (Zh3A1) at (4.5, -1) {$\widehat{\Zb}_{A_1}$};
    \node[var, fill=gray!30] (Zh3A2) at (6.75, -1) {$\widehat{\Zb}_{A_2}$};
    \draw[arrow] (Z31) -- (Z32); 
    \draw[arrow] (Z32) -- (Z33);
    \draw[-latex, thick] (Z31) to [out=45,in=145] (Z33);
    \draw[arrow] (Z31) -- (Zh3A1); 
    \draw[arrow] (Z32) -- (Zh3A2); 
    \draw[arrow] (Z33) -- (Zh3A2); 

    \coordinate (A) at (0.8, 1.5);
    \coordinate (B) at (3.8, 3.5);
    \coordinate (C) at (3.8, -0.5);

    \draw[-latex, ultra thick] (A) -- (B) node[midway, sloped, above] {Algorithm 1};
    \draw[-latex, ultra thick] (A) -- (C) node[midway, sloped, above] {Algorithm 2};

    \node[text width=1cm] at (6.3, 1.7) {(a)};
    \node[text width=1cm] at (6.3, -1.6) {(b)};

\end{tikzpicture}} 
  \end{center}
  \caption{(\emph{Left}) A measurement model where $\Xb$ is a fully mixed measurement of the causal variables. 
    $\Xb$ is often termed the \emph{observables} in CRL literature, representing the observed data. 
    (\emph{Right}) Two measurement models specified by different CRL identification algorithms: (a)
    Algorithm 1 guarantees one-to-one correspondence between the learned representation and causal variables; (b) Algorithm 2 guarantees that $\widehat\Zb_{A_1}$ corresponds to $\Zb_1$ while $\widehat\Zb_{A_2}$ represents a mixing of $\Zb_2$ and $\Zb_3$. 
    }
    \label{fig:measurement_models}
\vspace{-35pt}
\end{wrapfigure}
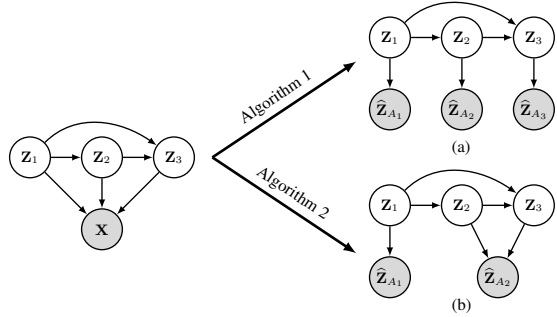
\begin{itemize}[leftmargin=*]
    \item \looseness=-1 We reinterpret CRL using a \emph{measurement model} framework, wherein the learned representations serve as proxy measurements for latent causal variables~\pcref{sec:measurement_model}. This formalism provides a clearer characterization of both the identification quality of learned representation and its usefulness for causal downstream tasks.
    \item \looseness=-1 We propose a new evaluation metric (T-MEX) that quantifies the alignment of the representations and the underlying measurement model~\pcref{sec:evaluation_metric}, and we demonstrate its advantages over widely used CRL evaluation metrics that suffer from notable limitations~\pcref{sec:problems_of_current_eval}.
    \item \looseness=-1 Supported by theoretical analysis, our empirical evaluations confirm that T-MEX maintains validity and effectiveness across diverse causal reasoning scenarios, including treatment effect estimation and covariate adjustment in both numerical simulations and real-world ecological experiments~\pcref{sec:exp}.
\end{itemize}

\section{CRL from A Measurement Model Perspective}
\label{sec:measurement_model}

\textbf{Notation.}
\looseness=-1 Throughout, we write $[N]$ as shorthand for the set $\{1,\dots,N\}$.  Random vectors are denoted by bold uppercase letters (e.g.\ $\Zb$) and their realizations by bold lowercase (e.g.,\ $\zb$), indexed by superscripts.  For instance, $n$ samples of $\Zb$ are written as $\{\zb^k\}_{k\in[N]}$.
A vector $\Zb$ can be sliced either by a single index $i \in [\dim(\Zb)]$ via $\Zb_i$ or a index subset $A \subseteq [\dim(\Zb)]$ with $\Zb_A := \{\Zb_i: i \in A\}$. 
$P_{\Zb}$ denotes the probability distribution of the random vector $\Zb$ and $p_{\Zb}(\zb)$ denotes the associated probability density function (We omit the subscription and write $p(\zb)$ when the context is clear). By default, a ``measurable'' function is \emph{measurable} w.r.t.\ the Borel sigma algebras and is defined w.r.t.\ the Lebesgue measure. A more comprehensive summary of notations is provided in~\cref{app:notations}. 

\subsection{The Measurement Model Framework}
We formulate causal representation learning using a measurement model framework inspired by the formalism of~\citep{silva2006learning}. 

\begin{definition}[Measurement model]\label{def:measurement_model}
Let $\Vb = (\Zb, \widehat{\Zb})$ be a collection of variables that can be partitioned into two sets: 
a set of latent \emph{causal variables} $\Zb = \{\Zb_1,\cdots, \Zb_N\}$ with $\Zb_i$ taking values in 
$\bbR$ for all $i\in[N]$, and a set of observed \emph{measurement variables} 
$\widehat{\Zb} = \{\widehat{\Zb}_{A_1},\cdots, \widehat{\Zb}_{A_M}\}$ where for all 
$j\in[M]$, $\widehat{\Zb}_{A_j}$ takes values in $\RR^{D_j}$ with $D_j \in\bbN_+$, 
and it holds that $\widehat\Zb\cap\Zb = \emptyset$.  

A \emph{measurement model} $\Mcal = \langle \Zb, \widehat{\Zb}, \{h_j\}_{j=1}^{M}\rangle$ 
specifies that $\widehat{\Zb}$ follows a deterministic structural causal model 
\begin{equation*}
\left\{\widehat{\Zb}_{A_j} \coloneqq h_j(\paHatZj)\right\}_{j=1}^{M}, 
\end{equation*}
where $\text{pa}(\widehat\Zb_{A_j})\subseteq [N]$ for all $j\in[M]$, and 
$\Zb_{\text{pa}(\widehat{\Zb}_{A_j})} \subseteq\Zb$ are called the causal parents of $\widehat{\Zb}_{A_j}$. 
The functions $h_j$ for all $j\in [M]$ are called the \textit{measurement functions}. 
If for some $j\in [M]$, $|\text{pa}(\widehat{\Zb}_{A_j})| = 1$ and the function $h_j$ is the identity map, 
then the causal variable $\text{pa}(\widehat{\Zb}_{A_j})$ is said to be \emph{measured directly}. 
\end{definition}

\begin{remark}[Difference from \citep{silva2006learning}]
While we borrow the concept of a measurement model from \citet{silva2006learning}, our framework differs in two key aspects. First, \citet{silva2006learning} aims to uncover relationships among latent causal variables by searching for pure measurements, i.e., a tree-structure in which latent nodes have fixed, noisy, low-dimensional observed children (measurements). 
In contrast, we interpret a given causal representation produced by a CRL algorithm as measurement variables and focus on evaluating their usefulness for specific causal tasks, which requires specification of a causal model. Second, \citet{silva2006learning} assumes a linear latent structural causal model, whereas our framework imposes no parametric structural assumption on the latent causal variables. Rather, we specify the relationship between the causal variables and their measurements according to certain hypotheses, such as identification guarantees, prior knowledge, or assumptions for specific causal downstream tasks. 
As we will see in \cref{sec:evaluation_metric}, this also allows us to properly evaluate a learned CRL model. 
\end{remark}

\begin{remark}
\label{rem:noisy_measurements}
    While we treat the measurement variables $\widehat{\Zb}$ as noise-free nonlinear mixing of their causal parents, we can easily extend our framework to noisy measurements by considering the noise variables as additional latent causal variables.
\end{remark}

\begin{example}
\label{ex:ident_measurement_model}
Assume by the identifiability theory of a specific CRL method that each $\widehat{\Zb}_{A_j}$ block-identifies~(see~\cref{def:identif} \citep[Defn~4.1]{von2021self}) a subset of latent variables $\Zb_{S_i}$ ($S_i \subseteq [N]$). Then for the measurement model $\Mcal = \langle \Zb, \widehat{\Zb}, \{h_j\}_{j=1}^{M} \rangle$ it holds that: $\widehat{\Zb}_{A_j} \coloneqq h_j(\Zb_{S_i}),$
with $h_j: \RR^{|S_i|} \to \RR^{D_j}$ a diffeomporphism for all $j \in [M]$. 

The measurement model induces a partial directed acyclic graph (DAG), that is, for any latent variable $q$ that is block-identified~\pcref{def:identif} by $A_j$, there is an edge from the latent causal variable 
$\Zb_q$ to the measurement variable $\widehat\Zb_{A_j}$, and the measurement function $h_j$ is a diffeomorphism. Illustrative examples are shown in~\cref{fig:measurement_models} for different identifiability guarantees.
\end{example}

\textbf{Discussion.} Note that a measurement model specified by certain identifiability theory (see ~\cref{fig:measurement_models}) is a necessary but not sufficient condition for drop-in replacement of a variable with its identified counterpart in a causal inference engine~\citep{pearl2018book} 
or a downstream causal estimand like \emph{average treatment effect}~\citep{robins1994estimation}. 
To this end, we introduce \emph{causally valid measurement model}.

\begin{definition}[Causally valid measurement model]
\label{def:causally_valid_model}
The measurement model~\pcref{def:measurement_model} is \textit{``causally valid" } with respect to a statistical estimand $g$ that identifies a target causal estimand, if the measurement $\widehat{\Zb}$ is a drop-in replacement in $g$ for the true causal variables $\Zb$, i.e., 
$g(\Zb) = g(\widehat{\Zb})$.
\end{definition}

\looseness=-1 \textbf{Discussion.}
Causal validity of a measurement model with respect to a specific estimand boils down to the estimand being invariant with respect to the measurement function.
As \citep{von2024nonparametric} already pointed out, identification of a latent causal variable up to a non-linear parameterization (i.e., block-identifiability~\pcref{def:identif}) does not allow average treatment effect estimation if either the treatment or outcome is a latent causal variable without additional information.
For that, a direct measurement (see~\cref{def:measurement_model}) as in~\citep{cadei2024smoke,cadei2025causal} is necessary; alternatively, one can choose an estimand that is invariant to non-linear invertible parameterizations, e.g., (conditional) mutual information~\citep{janzing2013quantifying}. 
As another example, a non-linear invertible parameterization is enough to model confounding variables~\citep{yao2024marrying} and instruments, see~\ref{app:causal_implication} for extended discussions and examples. 
Finally, note that the causal validity of the measurement models does not always require one-to-one correspondence between the measurement variables and latent causal variables: When an estimand concerns a coarse-graining of a subset of variables, then a measurement model mixing the right subset of variables can still be causally valid. For example, the valid adjustment set $\Wb$ in~\cref{fig:supervised_TY} can contain two or more variables, which can remain entangled with each other in the learned representation $\widehat{\Wb} := h(\Wb)$ as long as the measurement function $h$ is invertible, see~\cref{app:causal_implication} for detailed derivations.

\looseness=-1\textbf{When is a measurement model ``true''?} Note that any causal model between learned representation can always be trivially formulated as a measurement model, with each identified representation variable corresponding to a latent causal variable (i.e., $\widehat{\Zb}_1 \rightarrow \widehat{\Zb}_2$ implicitly implies a measurement model $\widehat{\Zb}_1\leftarrow \Zb_1 \rightarrow \Zb_2 \rightarrow \widehat{\Zb}_2$). 
Sometimes, by means of other assumptions, the latent causal model may not match one-to-one with the measurements; for example, see~\Cref{fig:measurement_models} (b). Our discussion on the measurement model only specifies the dependency between a learned representation and an (implicitly) assumed latent causal model. Following~\citep{peters2014causal}, we intend the latent causal model to be true if it agrees with the results of randomized studies in practice. If the latent causal model is true, then a causally valid measurement model is trivially also true.

\section{Evaluating Causal Representations using Measurement Models}
\label{sec:evaluation_metric}
\looseness=-1 This section explains how the measurement model formalism we introduced in~\cref{sec:measurement_model} 
serves as a natural tool to evaluate causal representations. 
A causal representation is defined as a set of measurement variables output from an encoder --- 
a parameterized function that maps the observables $\Xb$ to the measurement variables $\widehat{\Zb}$. 
Each CRL method specifies a measurement model, either through its identifiability guarantees or 
the particular causal task it addresses. This measurement model defines which causal variables a 
representation should \emph{exclusively measure}.
Given paired samples of the true causal variables $\Zb$ and their corresponding measurement variables $\widehat{\Zb}$ from a trained CRL model, evaluation boils down to comparing the measurement model against the observed joint distribution $P_{\Zb, \widehat{\Zb}}$.  
Before presenting our proposed evaluation metric, we introduce the following additional notation.
\begin{notation2}
Let $\Zb_1$, $\Zb_2$, and $\Zb_3$ be three absolutely continuous random variables taking values in 
$\bbR^{d_{Z_1}}$, $\bbR^{d_{Z_2}}$, and $\bbR^{d_{Z_3}}$ respectively. 
We say that $\Zb_1$ and $\Zb_2$ are \emph{conditionally independent} 
given $\Zb_3$ if $p(\Zb_1, \Zb_2\given \Zb_3) = p(\Zb_1\given \Zb_3) p(\Zb_2 \given \Zb_3)$, and it is denoted as $\Zb_1\indep\Zb_2\given \Zb_3$. 
A statistical test $\varphi$ is a function that maps data to $\{0,1\}$, e.g., 
$\varphi: \bbR^{n\times d_{Z_1}} \times \bbR^{n\times d_{Z_2}}\times \bbR^{n\times d_{Z_3}} \to\{0,1\}$, where $n$ denotes the number of samples.
The test $\varphi$ rejects a null hypothesis $\cH_0$ if $\varphi(\Zb_1,\Zb_2,\Zb_3) = 1$ and does not reject it if 
$\varphi(\Zb_1,\Zb_2,\Zb_3) = 0$.
Given a significance level $\alpha\in(0,1)$, a test is said to be \emph{valid} if it holds that 
$\sup_{P\in\cH_0} \bbP(\varphi(\Zb_1,\Zb_2,\Zb_3) = 1) \leq \alpha$, and 
it is said to have power $\beta\in(0,1)$ against an alternative distribution $P\not\in\cH_0$ if 
$\bbP(\varphi(\Zb_1,\Zb_2,\Zb_3) = 1) = \beta$. 
\end{notation2}




\looseness=-1 \textbf{Exclusivity of measurements.} 
A measurement model describes the relationship between the causal and the measurement 
variables. Specifically, it tell us for each measurement variable, which causal variables it 
should \emph{exclusively measure}. We formally define this concept below.

    
\begin{definition}[Exclusivity of a measurement variable]
\label{defn:exclusivity}
    Let $\Mcal = \langle\Zb, \widehat{\Zb}, \{h_j\}_{j \in [M]} \rangle$ be a measurement model, if a measurement variable $\widehat{\Zb}_{A_j}, j\in [M]$ only has one causal parent $\Zb_i$ for some $i \in [N]$, then we say $\widehat{\Zb}_{A_j}$ \emph{exclusively measures} $\Zb_i$. 
\end{definition}

Given samples of the causal and measurement variables 
$\{(\zb^k,\hat\zb^k)\}_{k\in[n]}$, 
we can check whether the measurement variables do satisfy the exclusivity property in the 
data by testing the following null hypotheses: 
\begin{equation} \label{eq:null_hypo}
    \cH_0(i, j): \widehat\Zb_{A_j}\indep \Zb_i\given \Zb_{[N]\setminus\{i\}},
\end{equation}
for all $i\in [N]$ and $j\in [M]$. For a numerical summary of the overall exclusivity 
of the measurement variables, we propose the following 
\emph{Test-based Measurement EXclusivity (T-MEX)} score. 



\begin{definition}[Test-based measurement exclusivity score] \label{def:tmex_score}
Let $V\in \{0,1\}^{N\times M}$ be the adjacency matrix corresponding to the conditional 
independencies according to a measurement model $\Mcal$, such that for all $j\in [M]$ and $i\in [N]$, 
$V_{ji} = 1$ if a causal variable $\Zb_i$ is a causal parent of a measurement variable $\widehat\Zb_{A_j}$ 
according to the measurement model, and $V_{ji} = 0$ otherwise. 
Let $\widehat{W}\in \{0,1\}^{N\times M}$ be the matrix constructed according to the 
test results of the conditional independencies in~\cref{eq:null_hypo} based on the samples 
of $(\Zb,\widehat\Zb)$, 
such that for all $j\in [M]$ and $i\in [N]$, $W_{ji} = 1$ if $\cH_0(i,j)$ is rejected, 
and $W_{ji} = 0$ otherwise. Then the test-based measurement exclusivity (T-MEX) score is defined 
as the \emph{hamming distance} 
between $V$ and $\widehat{W}$:
\begin{equation*}
\text{T-MEX}(V,\widehat{W}) \coloneqq \sum_{j=1}^M\sum_{i=1}^N \mathbbm{1}(V_{ji} \neq \widehat{W}_{ji}),
\end{equation*}
where $\mathbbm{1}$ denotes the indicator function.
\end{definition}

Details for computing T-MEX is given in~\cref{alg:comp_trex}.
As T-MEX score is based on conditional independence testing, 
its value depends on the randomness in the samples, and the properties of the statistical tests being used. 
In ~\cref{prop:trex_expected_score}, we show the upper bound of the expected T-MEX score when the joint distribution $P_{\Zb, \widehat{\Zb}}$
of the causal variables $\Zb$ and output measurement variables $\widehat{\Zb}$ from a CRL model does align with a 
measurement model. 


\begin{restatable}{proposition}{tmexExpectation}
\label{prop:trex_expected_score}
\looseness=-1 Let $\{\varphi_{ij}\}_{i\in[N], j\in[M]}$ be a family of tests for \cref{eq:null_hypo} where for all
$i\in[N]$ and $j\in [M]$, 
$\varphi_{ij}$ is valid with level $\alpha\in(0,1)$ and has power at least $\beta\in(0,1)$. 
Given an adjacency matrix $V\in\bbR^{N\times M}$ based on a measurement model, 
if the joint distribution $P_{\Zb, \widehat{\Zb}}$ of the causal and measurement variables does 
align with the measurement model, and each entry in $\widehat{W}$ is computed based on an 
independent set of samples $\{(\zb^k,\hat\zb^k)\}_{k\in[n_{ij}]}, n_{ij} \in \bbN_{+}$, then the expected T-MEX satisfies 
\begin{equation*}
   \bbE[\text{T-MEX}(V, \widehat{W})] \leq \alpha\cdot (MN - ||V||_1) + (1-\beta) \cdot ||V||_1,
\end{equation*}
where $||V||_1 = \sum_{i=1}^N\sum_{j=1}^M V_{ij}$ is the $L_1$-norm of $V$.
\end{restatable}


\begin{remark}
\cref{prop:trex_expected_score} assumes that each null hypothesis in~\cref{eq:null_hypo} is 
tested using an independent set of samples. When there is only one set of samples available for a large number of tests, using the same sample set can lead to inflation of the false positive rate, 
and may inflate the T-MEX score. 
In this case, we recommend doing a multiple comparison adjustment when constructing $\widehat{W}$, 
for example, the Bonferroni-Holm correction \citep{holm1979simple}, which controls the family-wise 
error rate while it does not make assumptions on the dependencies of the multiple p-values. 
\end{remark}

\begin{remark}
\looseness=-1In this section, we focus on the exclusivity perspective of a measurement model via an 
approach similar to the idea of falsification of causal graphs~\citep[e.g.,][]{kook2025falsifying, faller2024self}. 
This is a non-parametric approach which is agnostic to the measurement functions. 
In certain cases, however, a measurement model may contain not only the conditional independence 
structure, but also other parametric assumptions through specifications of the 
measurement functions $\{h_j\}_{j\in[M]}$. Then, one may extend T-MEX to also take these constraints into account. 
\end{remark}

\section{Related Work: Flaws of Existing Evaluation Metrics for CRL}\label{sec:problems_of_current_eval}
\looseness=-1In this section, we cover the metrics that have been used by most papers proposing causal representation learning approaches ~\citep{von2021self,von2024nonparametric,zheng2022identifiability,ahuja2024multi,ahuja2022weakly,varici2024general,zhang2024identifiability,zhang2024causal,yao2023multi,lippe2022causal,lippe2022citris,lachapelle2022disentanglement,lachapelle2022synergies,yao2022temporally,zhang2024identifiability,squires2023linear,buchholz2023learning,yao2024unifying} to name a few. We highlight how these metrics are not immediately suitable to evaluate identification results in the presence of causal relations, making it difficult to compare models and requiring great care in the interpretation of the results that is often missed~\citep{gamella2025sanity}. 

\looseness=-1 Standard evaluation for latent variable identification in existing CRL works employs \textit{coefficient of determination} $R^2$~\pcref{def:r2}, and \textit{mean correlation coefficient}~\pcref{def:mcc}. However, when the latent variables are causally related, a high score of these two metrics 
does not indicate that the learned representations align with the measurement model we expect from the identifiability theory.
\cref{ex:example0} illustrates this limitation of these two metrics under the presence of causal dependencies.


\begin{example} \label{ex:example0}
Assume that the latent causal variables ${\Zb}$ in~\cref{fig:measurement_models} (b) follow a linear Gaussian additive noise model.
Specifically, the latent variables $\Zb_1$ and $\Zb_2$ are generated based on the following structural 
equation: 
\begin{equation} \label{eq:struc_eq_cs}
    \Zb_2 \coloneqq a \cdot \Zb_1 + e
\end{equation}
with $e\sim P_e$, $\bbE[e] = 0$ and $e\indep \Zb_1$. 
Suppose that the measurement model which induces \cref{fig:measurement_models} (b) 
specifies that the measurement function $h: \RR \to \RR$ is a diffeomorphism such that 
$\widehat{\Zb}_{A_1} = h(\Zb_1)$, that is, 
$\widehat{\Zb}_{A_1}$ identifies $\Zb_1$, 
while $\widehat{\Zb}_{A_1}$ should not contain any additional information about $\Zb_2$.
\end{example}

\textbf{Coefficient of determination.} $R^2$ measures the proportion of the variation in the dependent variables explained by the regression model~\citep{draper1998applied}, formally defined as
\begin{definition}[Population $R^2$ score] \label{def:r2}
Let $(\Zb_i, \widehat{\Zb}_{A_j})$ be a pair of random variables both taking values in $\bbR$, $i \in [N], j \in [M]$.
The coefficient of determination $R^2$ score for predicting $\Zb_i$ from $\widehat{\Zb}_{A_j}$
is defined as  
\begin{equation*}
    R^2(\Zb_i, \widehat{\Zb}_{A_j}) \coloneqq \dfrac{\bbV(
    \bbE[\Zb_i~|~\widehat{\Zb}_{A_j}]
    )}{\bbV(\Zb_i)},
\end{equation*}
where $\bbE$ and $\bbV$ denote the expectation and variance operators, respectively. 
\end{definition}

\textbf{Problem of $R^2$ in~\cref{ex:example0}}: Let $R^2({\Zb_1}, \widehat{\Zb}_{A_1})$ denote the $R^2$ score as defined in~\cref{def:r2}. 
Following the linear mechanism in~\cref{eq:struc_eq_cs}, $R^2(\Zb_2, \widehat{\Zb}_{A_1})$ can be expressed as
\begin{equation}
    \begin{aligned}
        R^2(\Zb_2, \widehat{\Zb}_{A_1}) 
        &= \dfrac{\bbV(
        \bbE[\Zb_2~|~\widehat{\Zb}_{A_1}]
        )}{\bbV(\Zb_2)}
        = \dfrac{\bbV(
        \bbE[a\Zb_1 + e~|~\widehat{\Zb}_{A_1}]
        )}{\bbV(a\Zb_1 + e)} \\
        &= \dfrac{a^2\bbV(
        \bbE[\Zb_1~|~\widehat{\Zb}_{A_1}]
        )}{a^2\bbV(\Zb_1) + \bbV(e)}  
        = \dfrac{a^2 \bbV(\Zb_1)}{a^2 \bbV(\Zb_1) + \bbV(e)} R^2(\Zb_1,\widehat{\Zb}_{A_1}).
    \end{aligned}
\end{equation}
\looseness=-1
Depending on the noise level $\bbV(e)$, $R^2(\Zb_2, \widehat{\Zb}_{A_1})$ can be either close to $R^2(\Zb_1,\widehat{\Zb}_{A_1})$ when $\bbV(e) \ll a^2 \bbV(\Zb_1)$ or close to 0 when $\bbV(e)$ is significantly higher than $a^2 \bbV(\Zb_1)$; in either case it does not reflect whether $\widehat\Zb_{A_1}$ identifies $\Zb_2$ or not, in the sense of~\cref{def:identif}. Ultimately, $R^2$ is a metric for predictability, not for identifiability. Using it as an identifiability metric under causal dependency can lead to misinterpretation~\citep{gamella2025sanity}.

\begin{remark}[Other problems of $R^2$ score]
$R^2$ is designed to measure how well a \emph{linear} model fits between two random variables. 
When the fitted model is nonlinear, $R^2$ can yield values outside $[0,1]$, which can be misleading. See also~\citet{cameron1997r} for more details.
\end{remark}

\textbf{Mean correlation coefficient (MCC).} Intuitively, MCC measures the \textit{component-wise correspondence} between the learned representation $\widehat{\Zb}$ and the ground truth latent variables $\Zb$. 
When using MCC, it is required to have the same latent and encoding dimensions. 
We restate the definition of the MCC as follows.
\begin{definition}[Mean correlation coefficient] \label{def:mcc}
    \begin{equation*}
        \text{MCC} = \dfrac{1}{N} \max_{\pi \in \text{perm[N]}} \sum_{i=1}^N \vert \text{Corr} (\Zb_i, \widehat{\Zb}_{\pi(i)}) \vert ,
    \end{equation*}
where $\text{Corr}(\cdot,\cdot)$ refers to the Pearson correlation under linear relationship and Spearman correlation in the nonlinear case. 
\end{definition}
However, we notice that MCC cannot capture how well the representations are \textit{disentangled}, misaligning with its original purpose of measuring \textit{component-wise correspondence}. Assume in~\cref{fig:measurement_models} (b) that $\widehat{\Zb}_{A_1} = \widehat{\Zb}_1$ and $\widehat{\Zb}_{A_2} = [\widehat{\Zb}_2, \widehat{\Zb}_3]$. 
The learned representations $\widehat{\Zb}_{A_j}$ are linear mappings of their causal parents $\paHatZj$:
\begin{equation*}
        \widehat{\Zb}_{1} = s \cdot \Zb_1; \qquad
        \widehat{\Zb}_{2} = a \cdot \Zb_2 + b \cdot \Zb_3; \qquad
        \widehat{\Zb}_{3} = c \cdot \Zb_2 + d \cdot \Zb_3,
\end{equation*}
where $s, a, b, c, d \neq 0$. In this case, the MCC would obtain the highest value 1, although $\Zb_2, \Zb_3$ are still entangled in the learned representation $\hat{\Zb}$, demonstrating that MCC is inadequate in evaluating element-wise identification under causal relations. 

\textbf{Evaluation of causal relations.} Causal relations are usually evaluated with the standard metrics \emph{Structural Hamming distance} (SHD). We remark that evaluating causal discovery on the learned representations should always be done in conjunction with latent variable identification, as it is possible to achieve a perfect SHD (i.e, zero) with entangled representations, using e.g., LiNGAM~\citep{shimizu2006linear}, as shown numerically in~\cref{subsec:false_positive_shd}.

\textbf{Evaluation of disentangled representation.}
Evaluating disentangled representations (where the ground truth latent variables are assumed to be mutually independent) is comparatively easier. 
In the disentangled case, the main objective is to assess how well the learned representation aligns one-to-one with the ground truth latents. Commonly used evaluation metrics for disentangled representations include the  BetaVAE Score~\citep{higgins2017beta}, FactorVAE Score~\citep{kim2018disentangling}, Mutual Information Gap (MIG~\citet{chen2018isolating}),  DCI-disentanglement~\citep{eastwood2018framework}, Modularity~\citep{ridgeway2018learning} and SAP~\citep{kumar2017variational}. Broadly, evaluating learned representations can be viewed as a two-stage procedure, first estimating the relationship between latent variables and representations, and then aggregating this information into a single score~\citep{locatello2020sober}. In some way, our test can be seen as following the same strategy, although evaluating variable-level correspondence is less straightforward given underlying causal relationships, making it a fundamentally more challenging and understudied problem.


\section{Experiments} \label{sec:exp}
\looseness=-1 This section demonstrates the validity of the proposed T-MEX score in various causal reasoning settings. We first focus on \emph{covariate adjustment} in numerical simulations, using T-MEX to evaluate both identifiability~\pcref{def:identif} and causal validity~\pcref{def:causally_valid_model} of the representations~\pcref{subsec:simulated_exp}. Next, we move on to \emph{treatment effect estimation} in high-dimensional ecological video analysis, where we demonstrate that T-MEX effectively characterizes how well the learned representation supports answering downstream causal questions~\pcref{subsec:istant}. For both experiments, we estimate T-MEX based on the projected covariance 
measure (PCM) test \citep{lundborg2024projected} implemented in the python package \texttt{pycomets} \citep{huang_pycomets}, which is an algorithm-agnostic test for conditional independence (see~\cref{app:comets} for more explanations).
Further experiment details and additional results are deferred to~\cref{app:exp_details}.

\subsection{Numerical Simulation}
\label{subsec:simulated_exp}
This experiment validates our proposed T-MEX evaluation metric through a controlled numerical simulation. We leverage CRL to model confounders and perform backdoor adjustment to estimate the average treatment effect (ATE). We report both $R^2$ and the ATE bias, demonstrating that T-MEX closely tracks the ATE bias and provides a reliable measure of representation quality, whereas $R^2$ fails to yield consistent or meaningful conclusions.

\begin{figure}[htb]
\centering
\begin{minipage}{0.48\textwidth}
    \resizebox{0.9\linewidth}{!}{




\begin{tikzpicture}[loose/.style={inner sep=.7em},
    oval/.style={ellipse,draw}]
    \tikzstyle{var}=[circle, draw, thick, minimum size=8.5mm, font=\small, inner sep=1]
    \tikzstyle{arrow}=[-latex, thick]
    \tikzstyle{doublearrow}=[latex-latex, thick]
    \tikzstyle{dashedarrow}=[-latex, thick, dashed]




    \node[var] (Z1) at (-1.5, 0) {$\Zb_1$};
    \node[var] (Z2) at (0, 0) {$\Zb_2$};
    \node[var] (Z3) at (1.5, 0) {$\Zb_3$};
    \node[var, fill=gray!30] (Z4) at (3, 0) {$\Zb_4$};
    \node[var, fill=gray!30] (Z5) at (4.5, 0) {$\Zb_5$};

    \node[var, fill=gray!30] (ZhA1) at (-1.5, -1.5) {$\widehat\Zb_{A_1}$};
    \node[oval, thick, dashed, fit = (Z1) (Z2) (Z3), inner ysep=7pt, inner xsep=-7pt] (W) {};
    \node[inner sep=-2pt] (X) at (-1.2, 1.1) {$\Xb$};

    \draw[arrow] (Z1) -- (Z2); 
    \draw[arrow] (Z2) -- (Z3);
    \draw[arrow] (Z4) -- (Z5);
    \draw[-latex, thick] (Z1) to [out=-45,in=-145] (Z3);
    \draw[-latex, thick] (Z1) to [out=30,in=145] (Z4);
    \draw[-latex, thick] (Z1) to [out=45,in=145] (Z5);
    \draw[arrow] (Z1) -- (ZhA1);
    \draw[-latex, thick] (Z5) to [out=-145,in=-30] (Z2);

\end{tikzpicture}} 
    \captionof{figure}{Measurement model containing the \emph{latent} causal variables $\Zb_1$, $\Zb_2$, and $\Zb_3$ (white nodes) and 
\emph{observed} (also termed ``\emph{directly measured}" in~\cref{def:measurement_model}) causal variables $\Zb_4$ and $\Zb_5$ (gray nodes). 
The entangled observable $\Xb$ is shown as a dashed oval. $\widehat{\Zb}_{A_1}$ denotes the exclusive measurement~\pcref{defn:exclusivity} of $\Zb_1$.}
\label{fig:simu_measurement_model}
\end{minipage}
\hfill
\begin{minipage}{0.48\textwidth}
    \includegraphics[width=0.85\linewidth]{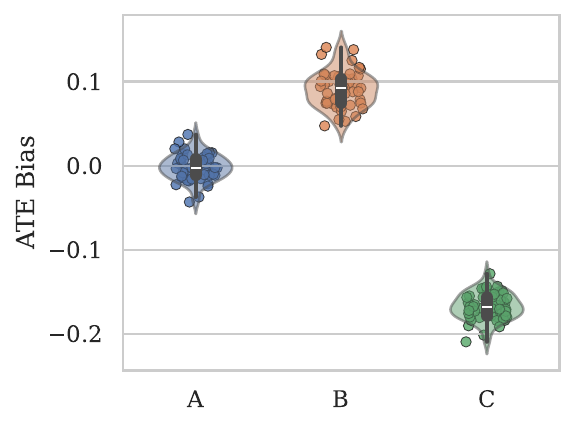}
    \captionof{figure}{\looseness=-1 \emph{T-MEX tracks the absolute bias of the ATE estimates} of $\Zb_4$ on $\Zb_5$ where $\widehat{\Zb}_1$ is conditioned on as the back door adjustment.}
    \label{fig:simu_ate_bias}
\end{minipage}
\end{figure}

\looseness=-1\textbf{Experiment settings.} 
We generate five causal variables, $\Zb_{i}$ for $i\in [5]$ according to a linear structural causal model (see \cref{app:numerical_simulation}), where two of the causal variables, $\Zb_4$ and $\Zb_5$, are \emph{observed} (also termed ``\emph{directly measured}" in~\cref{def:measurement_model}). 
The entangled observations $\Xb := f(\Zb_1, \Zb_2, \Zb_3)$ are generated by applying a diffeomorphism $f: \RR^3 \to \RR^3$, implemented as an invertible MLP, on the causal variables.
Our \emph{target causal task is to estimate the ATE of $\Zb_4$ on $\Zb_5$}. 
As the true causal relationship between $\Zb_4$ and $\Zb_5$ is linear, 
we can construct a consistent causal estimator where $\Zb_1$ is adjusted using linear regression, which is invariant up to \emph{bijective transformations} of $\Zb_1$~\pcref{app:causal_implication}.
Although $\Zb_1$ is latent and cannot be directly adjusted for, one can \emph{measure} it through a bijective transformation $\widehat{\Zb}_{A_1} := h(\Zb_1)$ which is obtained from the entangled observation $\Xb$. Note that in this case, $\widehat{\Zb}_{A_1}$ \emph{exclusively measures}~\pcref{defn:exclusivity} the confounder $\Zb_1$, as depicted in~\cref{fig:simu_measurement_model}.
We train three different CRL models based on the identifiable learning algorithm proposed by~\citet{yao2023multi} and obtain samples of the measurement variable $\widehat{\Zb}_{A_1}$:
\begin{itemize}
    \item \textbf{Model A}: a sufficiently trained model from which we expect the learned representation $\widehat{\Zb}_{A_1}^A$ (where by a slight abuse of notation, the superscript represents the model indicator) to \emph{exclusively measure} $\Zb_1$; \item \textbf{Model B}: an insufficiently trained model with unclear latent-measurement correspondence;
    \item \textbf{Model C}: a corrupted version of Model A where the representation $\widehat{\Zb}_{A_1}^C$ is defined as a linear mixing of the identified representation $\widehat{\Zb}_{A_1}^A$ and $\Zb_2, \Zb_3$.
\end{itemize}

\textbf{Results.} \cref{tab:res_simu} summarizes the T-MEX scores together with the coefficient of determination $R^2$ for all three models A, B and C, presented as \texttt{mean}$\pm$\texttt{sd}. \emph{For statistical validity}, we compute the results using $50$ simulated datasets
from each model, with each dataset containing $4096$ observations.
Further details about the test results are provided in~\cref{app:numerical_simulation}.
\Cref{tab:res_simu} shows that a sufficiently trained model (Model A) achieves a low T-MEX score, indicating that the learned representation $\widehat{\Zb}_{A_1}$ exclusively measures the latent variable $\Zb_1$. In contrast, the insufficiently trained and corrupted models (Models B and C) exhibit high T-MEX scores, demonstrating misalignment between the learned representation and the hypothesized measurement model (\cref{fig:simu_measurement_model}).
\Cref{fig:simu_ate_bias} presents the ATE bias estimated from the learned representations of all three models. We observe a strong correlation between T-MEX and the absolute bias of the ATE, validating T-MEX as a reliable indicator of the causal validity of the learned representation~\pcref{def:causally_valid_model}, whereas $R^2$ fails to show a clear correspondence with the ATE bias because $R^2$ was relatively high for all three latent variables as shown in~\cref{tab:res_simu}.

\begin{table}[hbt]
\centering
\caption{\looseness=-1 T-MEX, $R^2$ scores, and Spearman correlation coefficients of the learned representations (presented as \texttt{mean}$\pm$\texttt{std}) of \textbf{model A} (sufficiently trained, i.e., $\widehat{\Zb}_1$ exclusively measures $\Zb_1$), \textbf{model B} (insufficiently trained model with unclear latent-measurement correspondence) and \textbf{model C} (manually corrupted representation by linearly mixing $Z_2, Z_3$ with the representation of model A) based on $50$ simulated datasets, where each dataset contains $4096$ observations.}
\label{tab:res_simu}
\resizebox{\linewidth}{!}{
\begin{tabular}{@{}rrrrrrrr@{}}
\toprule
\multirow{2}{*}{\textbf{Model}}  & \multicolumn{1}{c}{\multirow{2}{*}{\textbf{T-MEX ($\downarrow$)}}} & \multicolumn{3}{c}{$\mathbf{R^2}$}  & \multicolumn{3}{c}{\textbf{Spearman Cor.~Coef.}}    \\ 
\cmidrule(lr){3-5} \cmidrule(lr){6-8} & \multicolumn{1}{c}{}  & $\Zb_1$ & $\Zb_2$ & $\Zb_3$ & $\Zb_1$ & $\Zb_2$ & $\Zb_3$ \\ 
\midrule
\textbf{A} & $0.1200\pm 0.3283$   & $0.9984 \pm 0.0001$   & $0.7516 \pm 0.0064$ & $0.8001 \pm 0.0006$ & $1.0000 \pm 0.0000$ & $0.8568 \pm 0.0044$ & $0.8864 \pm 0.0040$ \\
\textbf{B} & $1.1800 \pm 0.3881$ &  $0.6665 \pm 0.0078$  & $0.8305 \pm 0.0032$ &  $0.8707 \pm 0.0027$ & $0.8434 \pm 0.0061$ & $0.9602 \pm 0.0017$ & $0.9908 \pm 0.0004$\\ 
\textbf{C} & $2.0000 \pm 0.0000$ &  $0.9394 \pm 0.0016$  & $0.5421 \pm 0.0096$ &    $0.6627 \pm 0.0084$ & $0.9673\pm 0.0013$ & $0.7215 \pm 0.0076$ & $0.8016 \pm 0.0062$ \\ \bottomrule
\end{tabular}
}
\end{table}

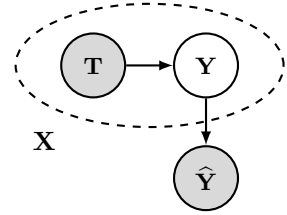
\begin{wrapfigure}{r}{0.4\textwidth}
\vspace{-40pt}
    \centering
    \resizebox{0.54\linewidth}{!}{\begin{tikzpicture}[loose/.style={inner sep=.7em},
    oval/.style={ellipse,draw}]
    \tikzstyle{var}=[circle, draw, thick, minimum size=8.5mm, font=\small, inner sep=1]
    \tikzstyle{arrow}=[-latex, thick]
    \tikzstyle{doublearrow}=[latex-latex, thick]
    \tikzstyle{dashedarrow}=[-latex, thick, dashed]


    \node[var, fill=gray!30] (T) at (0, 1.5) {$\Tb$};
    \node[var] (Y) at (1.5, 1.5) {$\Yb$};
    \node[var, fill=gray!30] (Yhat) at (1.5, 0) {$\widehat{\Yb}$};
    \node[oval, thick, dashed, fit = (T) (Y), inner ysep=4pt, inner xsep=2pt] (W) {};
    \node[inner sep=-2pt] (X) at (-0.65, 0.5) {$\Xb$};

    \draw[arrow] (Y) -- (Yhat); 
    \draw[arrow] (T) -- (Y);
\end{tikzpicture}}
    \caption{Measurement Model for the causal task in ISTAnt. $\Tb$ denotes the treatment (chemical exposure) and the \emph{latent} outcome $\Yb$ represents the ant's grooming behavior. Observable $\Xb$ (video recordings) is represented using a dashed oval. The measurement $\widehat{\Yb}$ \emph{exclusively measures}~\pcref{defn:exclusivity} $\Yb$.}
    \label{fig:istan_measurement_model}
    \vspace{-15pt}
\end{wrapfigure}

\subsection{Real-world Ecological Experiment: ISTAnt}
\label{subsec:istant}
\looseness=-1 This experiment validates the T-MEX score on ISTAnt~\citep{cadei2024smoke}, a real-world ecological benchmark designed for treatment effect estimation. We show a strong correlation between T-MEX and the absolute bias of the ATE, demonstrating that T-MEX can reliably evaluate the causal validity of learned representations under the challenge of high-dimensional real-world data.

\textbf{Experiment settings.} ISTAnt consists of video recordings of ant triplets with occasional grooming behavior. 
\emph{The goal is to extract a per-frame representation for supervised behavior classification (grooming or not) to estimate the ATE of an intervention (exposure to a certain pathogen)}.
Retrieving causally valid representations in this case is challenging as we have more non-annotated than annotated data, as described by \citep{cadei2024smoke}.
\Cref{fig:istan_measurement_model} depicts the hypothesized measurement model for this particular causal task, note that the treatment $\Tb$ and outcome $\Yb$ are unconfounded because the data is collected through a randomized controlled trial (RCT), meaning that the binary treatment $\Tb$ is randomly assigned.

\looseness=-1 \textbf{Results.} We compute the T-MEX score for 2,400 different models at a significance level of $\alpha=0.05$, and compare both classification accuracy and ATE bias against T-MEX. A full description of the considered models and training details is reported in~\cref{app:istant}. We only focus on the models that yield an accuracy over 80\% for meaningful statements. We observe that models with T-MEX = 0 achieve higher mean and lower variance for both accuracy and ATE bias, demonstrating that T-MEX effectively and reliably evaluates the quality of learned representations in terms of both classification performance and causal validity~\pcref{def:causally_valid_model}.

\looseness=-1 \textbf{Statistical validation.} To further assess the statistical significance between the T-MEX = 0 and T-MEX = 1 groups, we conduct a Mann-Whitney U test \citep{mann1947test} 
with the null hypothesis 
\begin{equation*}
    \Hcal_0: \mathbb{E}\Big[\vert\text{ATE Bias}\vert~\given~\text{T-MEX}=1\Big] \leq \mathbb{E}\Big[\vert\text{ATE Bias}\vert~\given~\text{T-MEX}=0\Big].
\end{equation*}
The resulting p-value of $0.0047$ leads us to reject $\Hcal_0$, 
providing strong evidence that the average absolute bias of the ATE for models with T-MEX = 1 is significantly higher than for those with T-MEX = 0.
Overall, T-MEX shows a strong correlation with the absolute bias of the ATE, validating its reliability as an evaluation metric for the causal validity of learned representations~\pcref{def:causally_valid_model}.

\textbf{Real-world implications of T-MEX.}
We emphasize that the proposed T-MEX score can be computed using only observational data, possibly with selection bias, as long as this selection bias does not change the conditional independence between measurements and causal variables. Instead, calculating the ATE bias as in~\citep{cadei2024smoke} requires a validation set that closely approximates the underlying population of the randomized controlled trial, a significantly stronger assumption that is often difficult to satisfy in real-world settings. Overall, T-MEX offers a convenient and accessible evaluation metric that reliably quantifies the usefulness of the learned representation for a causal downstream task, without the need for additional identifying assumptions.

\begin{figure}[htbp]
    \centering
    \includegraphics[width=\textwidth]{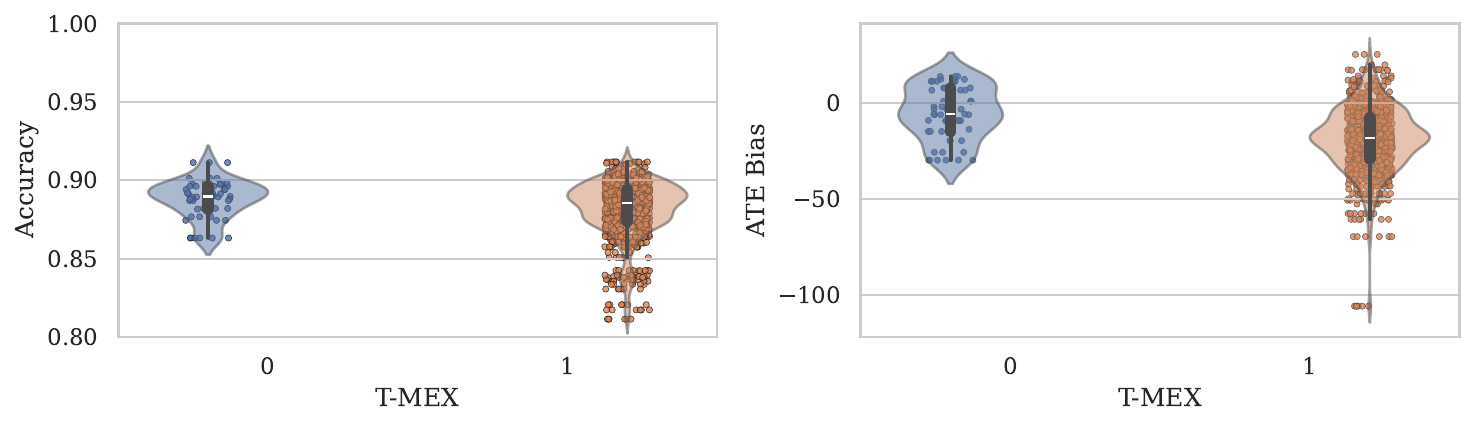}
    \caption{\emph{T-MEX reflects model performance in terms of both classification accuracy and causal validity~\pcref{def:causally_valid_model}}.
    Compared to their counterparts, models with lower T-MEX achieve consistently high accuracy (\emph{Left}) and center their ATE bias near zero with reduced variance (\emph{Right}).}
    \label{fig:exp}
\end{figure}

\subsection{Evaluation of T-MEX Properties using Synthetic Data}

In this section, we outline additional experiments designed to assess specific properties of T-MEX, including its scalability, robustness across different test choices, and behavior under weak or nonlinear causal relations as well as noisy measurements. The corresponding experimental details and results are provided in~\cref{app:add_exp}.


\textbf{Reliability and scalability with higher-dimensional latent variables.} It is well known that the statistical power of conditional independence (CI) tests deteriorates as the dimensionality of the conditioning set increases---a limitation shared by most CI methods \citep[e.g., ][]{shah2020hardness, strobl2019approximate, zhang2012kernel}. We examine the reliability and scalability of T-MEX given data generated from a \emph{non-linear location-scale SCM} with up to $50$ latent nodes in \cref{tab:T-MEX_higher_dims} in \cref{app:add_exp}. Since our proposed framework is agnostic to CI test choices, T-MEX can be easily scaled to larger dimensions by adapting more powerful and scalable test methods as they are developed.

\textbf{Consistency of T-MEX under different test choices}. CI testing is an important component in our framework, and choosing a valid and powerful test is essential. We chose PCM test in our experiments for its theoretical guarantees on the validity (type I error control) and power of the test, making it a reliable choice of test for the exclusivity claim. We also compare T-MEX scores using different conditional independence tests in \cref{tab:T-MEX_CI_tests} in \cref{app:add_exp}, based on the same experimental setup as in \cref{subsec:simulated_exp}. In this experiment, we see that T-MEX remains consistent across different test choices. 

\textbf{Consistency of T-MEX under noisy measurements.} 
As discussed in~\cref{rem:noisy_measurements}, noise in the measurement functions can be treated as additional independent latent variables. 
This means that if the noise is independent of the existing causal variables, the conditional independencies required for T-MEX remain intact. 
In such cases, one can safely compute T-MEX without explicitly modeling the noise; the score will remain valid. The empirical T-MEX value is shown to be robust under noise, as shown in~\cref{tab:T-MEX_noise} in \cref{app:add_exp}.

\textbf{T-MEX under weak causal relations.}
Building on the experiment in~\cref{subsec:simulated_exp}, we also examine how T-MEX behaves when the causal relationships between the latent causal variables are weak and compare it with the commonly used CRL identifiability metrics $R^2$ and MCC. See~\cref{tab:T-MEX_weak_causal} in \cref{app:add_exp}. In this case, T-MEX correctly reflects the entanglement in the representation, while both $R^2$ and MCC give a misleading score suggesting perfect element-wise identification. 

\textbf{Alignment between T-MEX and ATE under nonlinear SCM}.
Extending the results in~\cref{subsec:simulated_exp}, we further examine the correspondence between T-MEX and ATE under a nonlinear setting. As shown in~\cref{tab:t-mex_ate_nonlinear} in \cref{app:add_exp}, the results are highly similar to the linear case~\cref{subsec:simulated_exp}, T-MEX closely aligns with the absolute values of the ATE bias, effectively evaluating causal representations for downstream inference tasks with nonlinear causal relations.



\vspace{-10pt}
\section{Conclusion and Limitations}
\label{sec:conclusion_and_limitation}

\looseness=-1 This paper introduces a novel perspective on Causal Representation Learning (CRL) based on a measurement model framework, in which causal representations are treated as proxy measurements of latent causal variables~\pcref{sec:measurement_model}. 
This perspective provides a flexible framework that unites CRL identification theory with downstream task assumptions via measurement functions, yielding a principled way to evaluate representation quality.
More specifically, we propose a new evaluation metric, Test-based Measurement EXclusivity (T-MEX) score, which quantifies the discrepancy between a given measurement model (specified by a CRL algorithm, a causal task, or ground truth knowledge) and the joint distribution of causal and measurement variables (representation outputs of a CRL model) using conditional independence tests~\pcref{sec:evaluation_metric}. 
Because these conditional independence tests impose no parametric assumptions, our T-MEX score remains broadly applicable. However, like any statistical procedure, these tests are subject to sampling variability and potential statistical errors, so the reliability of T-MEX depends on which test is chosen. By remaining agnostic about the specific test, we empower practitioners to tailor the score to whatever assumptions they are willing to make (e.g., parametric or non-parametric).
We demonstrate, using both simulations~\pcref{subsec:simulated_exp} and real-world video analysis~\pcref{subsec:istant}, that our proposed T-MEX score effectively quantifies the identification and causal validity of the learned representation~\pcref{def:causally_valid_model}. This provides a convenient and practical evaluation scheme for representation quality in real-world scenarios, especially when the true treatment effect bias is unavailable, such as in the absence of randomized studies.

\subsubsection*{Acknowledgment}
This research was funded in whole or in part by the Austrian Science Fund (FWF) 10.55776/COE12. For open access purposes, the author has applied a CC BY public copyright license to any accepted manuscript version arising from this submission.


\clearpage
\bibliographystyle{abbrvnat}
\bibliography{refs} 

\begin{thebibliography}{74}
\providecommand{\natexlab}[1]{#1}
\providecommand{\url}[1]{\texttt{#1}}
\expandafter\ifx\csname urlstyle\endcsname\relax
  \providecommand{\doi}[1]{doi: #1}\else
  \providecommand{\doi}{doi: \begingroup \urlstyle{rm}\Url}\fi

\bibitem[Ahuja et~al.(2021)Ahuja, Caballero, Zhang, Gagnon-Audet, Bengio, Mitliagkas, and Rish]{ahuja2021invariance}
K.~Ahuja, E.~Caballero, D.~Zhang, J.-C. Gagnon-Audet, Y.~Bengio, I.~Mitliagkas, and I.~Rish.
\newblock Invariance principle meets information bottleneck for out-of-distribution generalization.
\newblock \emph{Advances in Neural Information Processing Systems}, 34:\penalty0 3438--3450, 2021.

\bibitem[Ahuja et~al.(2022)Ahuja, Hartford, and Bengio]{ahuja2022weakly}
K.~Ahuja, J.~S. Hartford, and Y.~Bengio.
\newblock Weakly supervised representation learning with sparse perturbations.
\newblock \emph{Advances in Neural Information Processing Systems}, 35:\penalty0 15516--15528, 2022.

\bibitem[Ahuja et~al.(2024)Ahuja, Mansouri, and Wang]{ahuja2024multi}
K.~Ahuja, A.~Mansouri, and Y.~Wang.
\newblock Multi-domain causal representation learning via weak distributional invariances.
\newblock In \emph{International Conference on Artificial Intelligence and Statistics}, pages 865--873. PMLR, 2024.

\bibitem[Ai et~al.(2024)Ai, Sun, Zhang, and Zhu]{ai2024testing}
C.~Ai, L.-H. Sun, Z.~Zhang, and L.~Zhu.
\newblock Testing unconditional and conditional independence via mutual information.
\newblock \emph{Journal of Econometrics}, 240\penalty0 (2):\penalty0 105335, 2024.

\bibitem[Arjovsky et~al.(2020)Arjovsky, Bottou, Gulrajani, and Lopez-Paz]{arjovsky2020invariantriskminimization}
M.~Arjovsky, L.~Bottou, I.~Gulrajani, and D.~Lopez-Paz.
\newblock Invariant risk minimization, 2020.

\bibitem[Bareinboim and Pearl(2016)]{bareinboim2016causal}
E.~Bareinboim and J.~Pearl.
\newblock Causal inference and the data-fusion problem.
\newblock \emph{Proceedings of the National Academy of Sciences}, 113\penalty0 (27):\penalty0 7345--7352, 2016.

\bibitem[Billingsley(2008)]{billingsley2008probability}
P.~Billingsley.
\newblock \emph{Probability and Measure}.
\newblock Wiley Series in Probability and Statistics. John Wiley \& Sons, Hoboken, NJ, 4th edition, 2008.

\bibitem[Buchholz et~al.(2024)Buchholz, Rajendran, Rosenfeld, Aragam, Sch{\"o}lkopf, and Ravikumar]{buchholz2023learning}
S.~Buchholz, G.~Rajendran, E.~Rosenfeld, B.~Aragam, B.~Sch{\"o}lkopf, and P.~Ravikumar.
\newblock Learning linear causal representations from interventions under general nonlinear mixing.
\newblock \emph{Advances in Neural Information Processing Systems}, 36, 2024.

\bibitem[Cadei et~al.(2024)Cadei, Lindorfer, Cremer, Schmid, and Locatello]{cadei2024smoke}
R.~Cadei, L.~Lindorfer, S.~Cremer, C.~Schmid, and F.~Locatello.
\newblock Smoke and mirrors in causal downstream tasks.
\newblock \emph{Advances in Neural Information Processing Systems}, 37, 2024.

\bibitem[Cadei et~al.(2025)Cadei, Demirel, De~Bartolomeis, Lindorfer, Cremer, Schmid, and Locatello]{cadei2025causal}
R.~Cadei, I.~Demirel, P.~De~Bartolomeis, L.~Lindorfer, S.~Cremer, C.~Schmid, and F.~Locatello.
\newblock Causal lifting of neural representations: Zero-shot generalization for causal inferences.
\newblock \emph{arXiv preprint arXiv:2502.06343}, 2025.

\bibitem[Cameron and Windmeijer(1997)]{cameron1997r}
A.~C. Cameron and F.~A. Windmeijer.
\newblock An {R-squared} measure of goodness of fit for some common nonlinear regression models.
\newblock \emph{Journal of econometrics}, 77\penalty0 (2):\penalty0 329--342, 1997.

\bibitem[Casella and Berger(2024)]{casella2024statistical}
G.~Casella and R.~Berger.
\newblock \emph{Statistical inference}.
\newblock CRC Press, 2024.

\bibitem[Chen et~al.(2018)Chen, Li, Grosse, and Duvenaud]{chen2018isolating}
R.~T. Chen, X.~Li, R.~B. Grosse, and D.~K. Duvenaud.
\newblock Isolating sources of disentanglement in variational autoencoders.
\newblock \emph{Advances in neural information processing systems}, 31, 2018.

\bibitem[Chernozhukov et~al.(2018)Chernozhukov, Chetverikov, Demirer, Duflo, Hansen, Newey, and Robins]{chernozhukov2018double}
V.~Chernozhukov, D.~Chetverikov, M.~Demirer, E.~Duflo, C.~Hansen, W.~Newey, and J.~Robins.
\newblock Double/debiased machine learning for treatment and structural parameters.
\newblock \emph{The Econometrics Journal}, 21\penalty0 (1):\penalty0 C1--C68, 2018.

\bibitem[Dong et~al.(2024)Dong, Huang, Ng, Song, Zheng, Jin, Legaspi, Spirtes, and Zhang]{dong2023versatile}
X.~Dong, B.~Huang, I.~Ng, X.~Song, Y.~Zheng, S.~Jin, R.~Legaspi, P.~Spirtes, and K.~Zhang.
\newblock A versatile causal discovery framework to allow causally-related hidden variables.
\newblock In \emph{The Twelfth International Conference on Learning Representations}, 2024.

\bibitem[Draper and Smith(1998)]{draper1998applied}
N.~R. Draper and H.~Smith.
\newblock \emph{Applied regression analysis}, volume 326.
\newblock John Wiley \& Sons, 1998.

\bibitem[Eastwood and Williams(2018)]{eastwood2018framework}
C.~Eastwood and C.~K. Williams.
\newblock A framework for the quantitative evaluation of disentangled representations.
\newblock In \emph{6th International Conference on Learning Representations}, 2018.

\bibitem[Faller et~al.(2024)Faller, Vankadara, Mastakouri, Locatello, and Janzing]{faller2024self}
P.~M. Faller, L.~C. Vankadara, A.~A. Mastakouri, F.~Locatello, and D.~Janzing.
\newblock Self-compatibility: Evaluating causal discovery without ground truth.
\newblock In \emph{International Conference on Artificial Intelligence and Statistics}, pages 4132--4140. PMLR, 2024.

\bibitem[Fern{\'a}ndez and Rivera(2024)]{fernandez2024general}
T.~Fern{\'a}ndez and N.~Rivera.
\newblock A general framework for the analysis of kernel-based tests.
\newblock \emph{Journal of Machine Learning Research}, 25\penalty0 (95):\penalty0 1--40, 2024.

\bibitem[Fumero et~al.(2024)Fumero, Wenzel, Zancato, Achille, Rodol{\`a}, Soatto, Sch{\"o}lkopf, and Locatello]{fumero2024leveraging}
M.~Fumero, F.~Wenzel, L.~Zancato, A.~Achille, E.~Rodol{\`a}, S.~Soatto, B.~Sch{\"o}lkopf, and F.~Locatello.
\newblock Leveraging sparse and shared feature activations for disentangled representation learning.
\newblock \emph{Advances in Neural Information Processing Systems}, 36, 2024.

\bibitem[Gamella et~al.(2025)Gamella, Bing, and Runge]{gamella2025sanity}
J.~L. Gamella, S.~Bing, and J.~Runge.
\newblock Sanity checking causal representation learning on a simple real-world system.
\newblock \emph{arXiv preprint arXiv:2502.20099}, 2025.

\bibitem[Higgins et~al.(2017)Higgins, Matthey, Pal, Burgess, Glorot, Botvinick, Mohamed, and Lerchner]{higgins2017beta}
I.~Higgins, L.~Matthey, A.~Pal, C.~Burgess, X.~Glorot, M.~Botvinick, S.~Mohamed, and A.~Lerchner.
\newblock {beta-VAE}: Learning basic visual concepts with a constrained variational framework.
\newblock In \emph{International conference on learning representations}, 2017.

\bibitem[Holm(1979)]{holm1979simple}
S.~Holm.
\newblock A simple sequentially rejective multiple test procedure.
\newblock \emph{Scandinavian journal of statistics}, pages 65--70, 1979.

\bibitem[Huang and Kook(2025)]{huang_pycomets}
S.~Huang and L.~Kook.
\newblock pycomets.
\newblock \href{https://github.com/shimenghuang/pycomets}{https://github.com/shimenghuang/pycomets}, 2025.
\newblock Accessed: 2025-05-15.

\bibitem[Hyv{\"a}rinen and Pajunen(1999)]{hyvarinen1999nonlinear}
A.~Hyv{\"a}rinen and P.~Pajunen.
\newblock Nonlinear independent component analysis: Existence and uniqueness results.
\newblock \emph{Neural networks}, 12\penalty0 (3):\penalty0 429--439, 1999.

\bibitem[Hyvarinen et~al.(2019)Hyvarinen, Sasaki, and Turner]{hyvarinen2019nonlinear}
A.~Hyvarinen, H.~Sasaki, and R.~Turner.
\newblock Nonlinear {ICA} using auxiliary variables and generalized contrastive learning.
\newblock In \emph{The 22nd International Conference on Artificial Intelligence and Statistics}, pages 859--868. PMLR, 2019.

\bibitem[Jain et~al.(2024)Jain, Denton, Whitfield, Didolkar, Earnshaw, Hartford, et~al.]{jain2024automated}
M.~Jain, A.~Denton, S.~Whitfield, A.~Didolkar, B.~Earnshaw, J.~Hartford, et~al.
\newblock Automated discovery of pairwise interactions from unstructured data.
\newblock \emph{arXiv preprint arXiv:2409.07594}, 2024.

\bibitem[Janzing et~al.(2013)Janzing, Balduzzi, Grosse-Wentrup, and Sch{\"o}lkopf]{janzing2013quantifying}
D.~Janzing, D.~Balduzzi, M.~Grosse-Wentrup, and B.~Sch{\"o}lkopf.
\newblock Quantifying causal influences.
\newblock \emph{The Annals of Statistics}, 41\penalty0 (5):\penalty0 2324--2358, 2013.

\bibitem[Kim and Mnih(2018)]{kim2018disentangling}
H.~Kim and A.~Mnih.
\newblock Disentangling by factorising.
\newblock In \emph{International conference on machine learning}, pages 2649--2658. PMLR, 2018.

\bibitem[Kong et~al.(2022)Kong, Xie, Yao, Zheng, Chen, Stojanov, Akinwande, and Zhang]{kong2022partial}
L.~Kong, S.~Xie, W.~Yao, Y.~Zheng, G.~Chen, P.~Stojanov, V.~Akinwande, and K.~Zhang.
\newblock Partial disentanglement for domain adaptation.
\newblock In \emph{International Conference on Machine Learning}, pages 11455--11472. PMLR, 2022.

\bibitem[Kong et~al.(2023)Kong, Huang, Xie, Xing, Chi, and Zhang]{kong2023identification}
L.~Kong, B.~Huang, F.~Xie, E.~Xing, Y.~Chi, and K.~Zhang.
\newblock Identification of nonlinear latent hierarchical models.
\newblock \emph{arXiv preprint arXiv:2306.07916}, 2023.

\bibitem[Kook(2025)]{kook2025falsifying}
L.~Kook.
\newblock Falsifying causal models via nonparametric conditional independence testing.
\newblock In \emph{Proceedings of the 39th International Workshop on Statistical Modelling (IWSM)}, 2025.
\newblock (to appear).

\bibitem[Kook and Lundborg(2024)]{kook2024algorithm}
L.~Kook and A.~R. Lundborg.
\newblock Algorithm-agnostic significance testing in supervised learning with multimodal data.
\newblock \emph{Briefings in Bioinformatics}, 25\penalty0 (6):\penalty0 bbae475, 2024.

\bibitem[Krueger et~al.(2021)Krueger, Caballero, Jacobsen, Zhang, Binas, Zhang, Le~Priol, and Courville]{krueger2021out}
D.~Krueger, E.~Caballero, J.-H. Jacobsen, A.~Zhang, J.~Binas, D.~Zhang, R.~Le~Priol, and A.~Courville.
\newblock Out-of-distribution generalization via risk extrapolation (rex).
\newblock In \emph{International conference on machine learning}, pages 5815--5826. PMLR, 2021.

\bibitem[Kumar et~al.(2017)Kumar, Sattigeri, and Balakrishnan]{kumar2017variational}
A.~Kumar, P.~Sattigeri, and A.~Balakrishnan.
\newblock Variational inference of disentangled latent concepts from unlabeled observations.
\newblock \emph{arXiv preprint arXiv:1711.00848}, 2017.

\bibitem[Lachapelle et~al.(2022)Lachapelle, Pau, Sharma, Everett, {Le Priol}, Lacoste, and Lacoste-Julien]{lachapelle2022disentanglement}
S.~Lachapelle, R.~Pau, Y.~Sharma, K.~E. Everett, R.~{Le Priol}, A.~Lacoste, and S.~Lacoste-Julien.
\newblock Disentanglement via mechanism sparsity regularization: A new principle for nonlinear {ICA}.
\newblock In \emph{First Conference on Causal Learning and Reasoning}, 2022.

\bibitem[Lachapelle et~al.(2023)Lachapelle, Deleu, Mahajan, Mitliagkas, Bengio, Lacoste-Julien, and Bertrand]{lachapelle2022synergies}
S.~Lachapelle, T.~Deleu, D.~Mahajan, I.~Mitliagkas, Y.~Bengio, S.~Lacoste-Julien, and Q.~Bertrand.
\newblock Synergies between disentanglement and sparsity: Generalization and identifiability in multi-task learning.
\newblock In \emph{International Conference on Machine Learning}, pages 18171--18206. PMLR, 2023.

\bibitem[Lippe et~al.(2022{\natexlab{a}})Lippe, Magliacane, L{\"o}we, Asano, Cohen, and Gavves]{lippe2022causal}
P.~Lippe, S.~Magliacane, S.~L{\"o}we, Y.~M. Asano, T.~Cohen, and E.~Gavves.
\newblock Causal representation learning for instantaneous and temporal effects in interactive systems.
\newblock In \emph{The Eleventh International Conference on Learning Representations}, 2022{\natexlab{a}}.

\bibitem[Lippe et~al.(2022{\natexlab{b}})Lippe, Magliacane, L{\"o}we, Asano, Cohen, and Gavves]{lippe2022citris}
P.~Lippe, S.~Magliacane, S.~L{\"o}we, Y.~M. Asano, T.~Cohen, and S.~Gavves.
\newblock Citris: Causal identifiability from temporal intervened sequences.
\newblock In \emph{International Conference on Machine Learning}, pages 13557--13603. PMLR, 2022{\natexlab{b}}.

\bibitem[Locatello et~al.(2019)Locatello, Bauer, Lucic, Raetsch, Gelly, Sch{\"o}lkopf, and Bachem]{locatello2019challenging}
F.~Locatello, S.~Bauer, M.~Lucic, G.~Raetsch, S.~Gelly, B.~Sch{\"o}lkopf, and O.~Bachem.
\newblock Challenging common assumptions in the unsupervised learning of disentangled representations.
\newblock In \emph{International Conference on Machine Learning}, pages 4114--4124. PMLR, 2019.

\bibitem[Locatello et~al.(2020)Locatello, Bauer, Lucic, R{\"a}tsch, Gelly, Sch{\"o}lkopf, and Bachem]{locatello2020sober}
F.~Locatello, S.~Bauer, M.~Lucic, G.~R{\"a}tsch, S.~Gelly, B.~Sch{\"o}lkopf, and O.~Bachem.
\newblock A sober look at the unsupervised learning of disentangled representations and their evaluation.
\newblock \emph{Journal of Machine Learning Research}, 21\penalty0 (209):\penalty0 1--62, 2020.

\bibitem[Lundborg et~al.(2024)Lundborg, Kim, Shah, and Samworth]{lundborg2024projected}
A.~R. Lundborg, I.~Kim, R.~D. Shah, and R.~J. Samworth.
\newblock The projected covariance measure for assumption-lean variable significance testing.
\newblock \emph{The Annals of Statistics}, 52\penalty0 (6):\penalty0 2851--2878, 2024.

\bibitem[Mann and Whitney(1947)]{mann1947test}
H.~B. Mann and D.~R. Whitney.
\newblock On a test of whether one of two random variables is stochastically larger than the other.
\newblock \emph{The Annals of Mathematical Statistics}, pages 50--60, 1947.

\bibitem[Oquab et~al.(2023)Oquab, Darcet, Moutakanni, Vo, Szafraniec, Khalidov, Fernandez, Haziza, Massa, El-Nouby, et~al.]{oquab2023dinov2}
M.~Oquab, T.~Darcet, T.~Moutakanni, H.~Vo, M.~Szafraniec, V.~Khalidov, P.~Fernandez, D.~Haziza, F.~Massa, A.~El-Nouby, et~al.
\newblock Dinov2: Learning robust visual features without supervision.
\newblock \emph{arXiv preprint arXiv:2304.07193}, 2023.

\bibitem[Pearl and Mackenzie(2018)]{pearl2018book}
J.~Pearl and D.~Mackenzie.
\newblock \emph{The book of why: the new science of cause and effect}.
\newblock Basic books, 2018.

\bibitem[Peters et~al.(2014)Peters, Mooij, Janzing, and Sch{\"o}lkopf]{peters2014causal}
J.~Peters, J.~M. Mooij, D.~Janzing, and B.~Sch{\"o}lkopf.
\newblock Causal discovery with continuous additive noise models.
\newblock \emph{Journal of Machine Learning Research}, 2014.

\bibitem[Peters et~al.(2017)Peters, Janzing, and Schlkopf]{peters2017elements}
J.~Peters, D.~Janzing, and B.~Schlkopf.
\newblock \emph{Elements of Causal Inference: Foundations and Learning Algorithms}.
\newblock The MIT Press, 2017.
\newblock ISBN 0262037319.

\bibitem[Ravuri et~al.(2025)Ravuri, Ulicna, Osea, Donhauser, and Hartford]{ravuri2025weakly}
A.~Ravuri, K.~Ulicna, J.~Osea, K.~Donhauser, and J.~Hartford.
\newblock Weakly supervised latent variable inference of proximity bias in crispr gene knockouts from single-cell images.
\newblock In \emph{Learning Meaningful Representations of Life (LMRL) Workshop at ICLR 2025}, 2025.

\bibitem[Ridgeway and Mozer(2018)]{ridgeway2018learning}
K.~Ridgeway and M.~C. Mozer.
\newblock Learning deep disentangled embeddings with the {F-statistic} loss.
\newblock \emph{Advances in neural information processing systems}, 31, 2018.

\bibitem[Robins et~al.(1994)Robins, Rotnitzky, and Zhao]{robins1994estimation}
J.~M. Robins, A.~Rotnitzky, and L.~P. Zhao.
\newblock Estimation of regression coefficients when some regressors are not always observed.
\newblock \emph{Journal of the American statistical Association}, 89\penalty0 (427):\penalty0 846--866, 1994.

\bibitem[Rojas-Carulla et~al.(2018)Rojas-Carulla, Sch{\"o}lkopf, Turner, and Peters]{rojas2018invariant}
M.~Rojas-Carulla, B.~Sch{\"o}lkopf, R.~Turner, and J.~Peters.
\newblock Invariant models for causal transfer learning.
\newblock \emph{Journal of Machine Learning Research}, 19\penalty0 (36):\penalty0 1--34, 2018.

\bibitem[Runge(2018)]{runge18conditional}
J.~Runge.
\newblock Conditional independence testing based on a nearest-neighbor estimator of conditional mutual information.
\newblock In A.~Storkey and F.~Perez-Cruz, editors, \emph{Proceedings of the Twenty-First International Conference on Artificial Intelligence and Statistics}, volume~84 of \emph{Proceedings of Machine Learning Research}, pages 938--947. PMLR, 09--11 Apr 2018.

\bibitem[Sch{\"o}lkopf et~al.(2021)Sch{\"o}lkopf, Locatello, Bauer, Ke, Kalchbrenner, Goyal, and Bengio]{scholkopf2021toward}
B.~Sch{\"o}lkopf, F.~Locatello, S.~Bauer, N.~R. Ke, N.~Kalchbrenner, A.~Goyal, and Y.~Bengio.
\newblock Toward causal representation learning.
\newblock \emph{Proceedings of the IEEE}, 109\penalty0 (5):\penalty0 612--634, 2021.

\bibitem[Shah and Peters(2020)]{shah2020hardness}
R.~D. Shah and J.~Peters.
\newblock The hardness of conditional independence testing and the generalised covariance measure.
\newblock \emph{The Annals of Statistics}, 48\penalty0 (3):\penalty0 1514--1538, 2020.

\bibitem[Shimizu et~al.(2006)Shimizu, Hoyer, Hyv{\"a}rinen, Kerminen, and Jordan]{shimizu2006linear}
S.~Shimizu, P.~O. Hoyer, A.~Hyv{\"a}rinen, A.~Kerminen, and M.~Jordan.
\newblock A linear non-gaussian acyclic model for causal discovery.
\newblock \emph{Journal of Machine Learning Research}, 7\penalty0 (10), 2006.

\bibitem[Silva et~al.(2006)Silva, Scheines, Glymour, Spirtes, and Chickering]{silva2006learning}
R.~Silva, R.~Scheines, C.~Glymour, P.~Spirtes, and D.~M. Chickering.
\newblock Learning the structure of linear latent variable models.
\newblock \emph{Journal of Machine Learning Research}, 7\penalty0 (2), 2006.

\bibitem[Squires et~al.(2023)Squires, Seigal, Bhate, and Uhler]{squires2023linear}
C.~Squires, A.~Seigal, S.~S. Bhate, and C.~Uhler.
\newblock Linear causal disentanglement via interventions.
\newblock In \emph{International Conference on Machine Learning}, volume 202, pages 32540--32560. {PMLR}, 2023.

\bibitem[Strobl et~al.(2019)Strobl, Zhang, and Visweswaran]{strobl2019approximate}
E.~V. Strobl, K.~Zhang, and S.~Visweswaran.
\newblock Approximate kernel-based conditional independence tests for fast non-parametric causal discovery.
\newblock \emph{Journal of Causal Inference}, 7\penalty0 (1):\penalty0 20180017, 2019.

\bibitem[Sun et~al.(2025)Sun, Kong, Chen, Li, Luo, Li, Zhang, Zheng, Yang, Stojanov, et~al.]{sun2025causal}
Y.~Sun, L.~Kong, G.~Chen, L.~Li, G.~Luo, Z.~Li, Y.~Zhang, Y.~Zheng, M.~Yang, P.~Stojanov, et~al.
\newblock Causal representation learning from multimodal biomedical observations.
\newblock In \emph{The Thirteenth International Conference on Learning Representations}, 2025.

\bibitem[Varici et~al.(2024)Varici, Acart{\"u}rk, Shanmugam, and Tajer]{varici2024general}
B.~Varici, E.~Acart{\"u}rk, K.~Shanmugam, and A.~Tajer.
\newblock General identifiability and achievability for causal representation learning.
\newblock In \emph{International Conference on Artificial Intelligence and Statistics}, pages 2314--2322. PMLR, 2024.

\bibitem[von K{\"u}gelgen et~al.(2021)von K{\"u}gelgen, Sharma, Gresele, Brendel, Sch{\"o}lkopf, Besserve, and Locatello]{von2021self}
J.~von K{\"u}gelgen, Y.~Sharma, L.~Gresele, W.~Brendel, B.~Sch{\"o}lkopf, M.~Besserve, and F.~Locatello.
\newblock Self-supervised learning with data augmentations provably isolates content from style.
\newblock \emph{Advances in Neural Information Processing Systems}, 34:\penalty0 16451--16467, 2021.

\bibitem[von K{\"u}gelgen et~al.(2024)von K{\"u}gelgen, Besserve, Wendong, Gresele, Keki{\'c}, Bareinboim, Blei, and Sch{\"o}lkopf]{von2024nonparametric}
J.~von K{\"u}gelgen, M.~Besserve, L.~Wendong, L.~Gresele, A.~Keki{\'c}, E.~Bareinboim, D.~Blei, and B.~Sch{\"o}lkopf.
\newblock Nonparametric identifiability of causal representations from unknown interventions.
\newblock \emph{Advances in Neural Information Processing Systems}, 36, 2024.

\bibitem[Wendong et~al.(2024)Wendong, Keki{\'c}, von K{\"u}gelgen, Buchholz, Besserve, Gresele, and Sch{\"o}lkopf]{liang2023causal}
L.~Wendong, A.~Keki{\'c}, J.~von K{\"u}gelgen, S.~Buchholz, M.~Besserve, L.~Gresele, and B.~Sch{\"o}lkopf.
\newblock Causal component analysis.
\newblock \emph{Advances in Neural Information Processing Systems}, 36, 2024.

\bibitem[Xie et~al.(2024)Xie, Huang, Chen, Cai, Glymour, Geng, and Zhang]{xie2024generalized}
F.~Xie, B.~Huang, Z.~Chen, R.~Cai, C.~Glymour, Z.~Geng, and K.~Zhang.
\newblock Generalized independent noise condition for estimating causal structure with latent variables.
\newblock \emph{Journal of Machine Learning Research}, 25\penalty0 (191):\penalty0 1--61, 2024.

\bibitem[Yao et~al.(2024{\natexlab{a}})Yao, Muller, and Locatello]{yao2024marrying}
D.~Yao, C.~Muller, and F.~Locatello.
\newblock Marrying causal representation learning with dynamical systems for science.
\newblock \emph{Advances in Neural Information Processing Systems}, 37, 2024{\natexlab{a}}.

\bibitem[Yao et~al.(2024{\natexlab{b}})Yao, Xu, Lachapelle, Magliacane, Taslakian, Martius, von K{\"u}gelgen, and Locatello]{yao2023multi}
D.~Yao, D.~Xu, S.~Lachapelle, S.~Magliacane, P.~Taslakian, G.~Martius, J.~von K{\"u}gelgen, and F.~Locatello.
\newblock Multi-view causal representation learning with partial observability.
\newblock In \emph{The Twelfth International Conference on Learning Representations}, 2024{\natexlab{b}}.

\bibitem[Yao et~al.(2025)Yao, Rancati, Cadei, Fumero, and Locatello]{yao2024unifying}
D.~Yao, D.~Rancati, R.~Cadei, M.~Fumero, and F.~Locatello.
\newblock Unifying causal representation learning with the invariance principle.
\newblock \emph{The Thirteenth International Conference on Learning Representations}, 2025.

\bibitem[Yao et~al.(2022)Yao, Chen, and Zhang]{yao2022temporally}
W.~Yao, G.~Chen, and K.~Zhang.
\newblock Temporally disentangled representation learning.
\newblock \emph{Advances in Neural Information Processing Systems}, 35:\penalty0 26492--26503, 2022.

\bibitem[Zhang et~al.(2024{\natexlab{a}})Zhang, Greenewald, Squires, Srivastava, Shanmugam, and Uhler]{zhang2024identifiability}
J.~Zhang, K.~Greenewald, C.~Squires, A.~Srivastava, K.~Shanmugam, and C.~Uhler.
\newblock Identifiability guarantees for causal disentanglement from soft interventions.
\newblock \emph{Advances in Neural Information Processing Systems}, 36, 2024{\natexlab{a}}.

\bibitem[Zhang et~al.(2012)Zhang, Peters, Janzing, and Sch{\"o}lkopf]{zhang2012kernel}
K.~Zhang, J.~Peters, D.~Janzing, and B.~Sch{\"o}lkopf.
\newblock Kernel-based conditional independence test and application in causal discovery.
\newblock \emph{arXiv preprint arXiv:1202.3775}, 2012.

\bibitem[Zhang et~al.(2020)Zhang, Gong, Stojanov, Huang, Liu, and Glymour]{zhang2020domain}
K.~Zhang, M.~Gong, P.~Stojanov, B.~Huang, Q.~Liu, and C.~Glymour.
\newblock Domain adaptation as a problem of inference on graphical models.
\newblock \emph{Advances in Neural Information Processing Systems}, 33, 2020.

\bibitem[Zhang et~al.(2024{\natexlab{b}})Zhang, Xie, Ng, and Zheng]{zhang2024causal}
K.~Zhang, S.~Xie, I.~Ng, and Y.~Zheng.
\newblock Causal representation learning from multiple distributions: A general setting.
\newblock \emph{Internatinal Conference on Machine Learning}, 2024{\natexlab{b}}.

\bibitem[Zheng et~al.(2022)Zheng, Ng, and Zhang]{zheng2022identifiability}
Y.~Zheng, I.~Ng, and K.~Zhang.
\newblock On the identifiability of nonlinear {ICA}: Sparsity and beyond.
\newblock \emph{Advances in neural information processing systems}, 35:\penalty0 16411--16422, 2022.

\bibitem[Zheng et~al.(2024)Zheng, Huang, Chen, Ramsey, Gong, Cai, Shimizu, Spirtes, and Zhang]{zheng2024causal}
Y.~Zheng, B.~Huang, W.~Chen, J.~Ramsey, M.~Gong, R.~Cai, S.~Shimizu, P.~Spirtes, and K.~Zhang.
\newblock Causal-learn: Causal discovery in python.
\newblock \emph{Journal of Machine Learning Research}, 25\penalty0 (60):\penalty0 1--8, 2024.

\end{thebibliography}



\clearpage
\appendix
\addcontentsline{toc}{section}{Appendix} 
\part{Appendix} 
{
  \hypersetup{linkcolor=DarkGreen}
  \parttoc 
}
\makenomenclature
\renewcommand{\nomname}{} 

\section{Notation and Terminology}
\label{app:notations}

This section summarizes the symbols used throughout the paper.
\vspace{-2em}
\nomenclature[0]{\(\Zb\)}{Causal variables}
\nomenclature[1]{\(\Xb\)}{Observables}
\nomenclature[2]{\(N\)}{Dimension of the causal variables $\Zb$}
\nomenclature[2]{\(D_j\)}{Dimension of the representation $\widehat{\Zb}_{A_j}$}
\nomenclature[3]{\(n\)}{Number of samples for the statistical tests for T-MEX}
\nomenclature[5]{\(\bbP(\cdot)\)}{Probability operator}
\nomenclature[6]{\(\bbE(\cdot)\)}{Expectation operator}
\nomenclature[7]{\(\mathbbm{1}(\cdot)\)}{Indicator function}

\printnomenclature

\section{Preliminaries}
\label{app:preliminaries}

\begin{definition}[Block-identifiability~\citep{von2021self}] \label{def:identif}
A set of latent variables $\Zb \in \bbR^{d_z}$ is block-identified by a representation $\widehat{\Zb} \in \bbR^{d_{\hat{z}}}$ if there exists a bijection 
$h:\bbR^{d_z}\to\bbR^{d_{\hat{z}}}$ such that $\widehat{\Zb} = h(\Zb)$. 
\end{definition}


\section{Proofs and Algorithms}
\label{app:proof}

This section includes the proof for~\cref{prop:trex_expected_score} and the algorithm to compute the T-MEX score.

\SetKwInput{KwData}{Input}
\SetKwInput{KwResult}{Output}
\DontPrintSemicolon
\begin{algorithm}[!htb]
\caption{Compute T-MEX score from one set of samples} \label{alg:comp_trex}
\KwData{Paired samples of causal variables and measurement variables 
$\{\zb,\widehat{\zb}_{A_1},\ldots,\widehat{\zb}_{A_M}\}$ where $\zb\in\bbR^{n \times N}$ and 
$\widehat{\zb}_{A_j}\in\bbR^{n\times D_j}$ for $j\in [M]$, 
adjacency matrix of the measurement model $V\in \{0,1\}^{N\times M}$, 
a set of statistical tests for $\{\varphi_{ij}\}_{i\in[N],j\in[M]}$ for \eqref{eq:null_hypo}, 
where for all $i\in [N], j\in [M]$, $\varphi_{ij}: \bbR^{n \times 1} \times\bbR^{n\times D_j}\times\bbR^{n\times (N-1)}\to\{0,1\}$}
\KwResult{T-MEX score of the given sample}
$\widehat{W} \gets \mathbf{0}^{N \times M}$ \; 
\For{$i \in [N]$}{
   \For{$j \in [M]$} {
    $\widehat{W}_{ij} \gets \varphi_{ij}(\zb_i, \widehat{\zb}_{A_j}, \zb_{[N]\setminus\{i\}})$ \; 
   }
}
\Return $\sum_{i=1}^N \sum_{j=1}^M\mathbbm{1}(V_{ij} \neq \widehat{W}_{ij})$   
\end{algorithm}

\tmexExpectation*
\begin{proof}
Suppose the joint distribution of $(\Zb, \widehat\Zb)$ aligns with the conditional 
independencies indicated by the adjacency matrix $V$, that is, for all $i\in[N]$ and $j\in[M]$, 
if $V_{ij} = 0$, it holds that $\widehat\Zb_{A_j}\indep \Zb_i\given \Zb_{[N]\setminus\{i\}}$; 
if $V_{ij} = 1$, it holds that $\widehat\Zb_{A_j}\dep \Zb_i\given \Zb_{[N]\setminus\{i\}}$.  

Fix a significance level $\alpha \in(0,1)$. Suppose for all $i\in[N]$ and all $j\in[M]$, 
the statistical test $\varphi_{ij}: \bbR^{n\times 1}\times \bbR^{n\times D_j}\times \bbR^{n\times (N-1)} \to \{0,1\}$ is valid at level $\alpha$ and has power at least 
$\beta\in[0,1]$ against the alternative distribution where $\widehat\Zb_{A_j}\dep \Zb_i\given \Zb_{[N]\setminus\{i\}}$. 

Then given independent sets of samples $\{\zb^k,\hat\zb^k\}_{k\in [n_{ij}]}$ for $i\in[N]$ and $j\in[M]$, and $\widehat{W}_{ij} = \varphi_{ij}(\zb_i, \hat\zb_{A_j}, \zb_{[N]\setminus\{i\}})$, 
it holds that 
\begin{itemize}
    \item if $V_{ij} = 0$, then $P(\widehat{W}_{ij} = 1) \leq \alpha$;
    \item if $V_{ij} = 1$, then $P(\widehat{W}_{ij} = 0) \leq 1-\beta$. 
\end{itemize}
Therefore, the expected value of T-MEX score is given by 
\begin{align*}
    \bbE[T\-MEX(V,\widehat{W})] &= \alpha \cdot \sum_{i,j} \mathbbm{1}(V_{ij} = 0) + (1 - \beta) \sum_{i, j} \cdot \mathbbm{1}(V_{ij} = 1) \\ 
    &\leq \alpha\cdot (MN - ||V||_1) + (1-\beta) \cdot ||V||_1
\end{align*}
where $||V||_1$ is the 1-norm of $V$. The second inequality is implied by the that each test 
$\varphi_{ij}$ is valid with level $\alpha$ and has power $\geq \beta$. 
\end{proof}

\begin{remark}
Proposition~\ref{prop:trex_expected_score} tells us that if the measurement 
model does hold for the joint distribution of the causal variables and the output representations 
from a trained CRL model, we would expect to see a ``low" T-MEX score given that we 
employ valid statistical tests that are also powerful enough to reject the null under 
alternatives. A ``low" T-MEX score does not in general refer to a $0$ score, as it depends on $V$, 
the chosen significance level $\alpha$, and the power of the test $\beta$. 
For example, let $\alpha=0.05$, we consider a valid statistical test that has the highest power, i.e., $\beta = 1$, additionally, assume the number of $0$s in $V$ is $2$, 
then the expected value of the T-MEX score is no larger than $0.05\times 2 = 0.1$. 
\end{remark}



    


\section{Experiment Details and Additional Results}
\label{app:exp_details}
This section elaborates on the experiment settings of~\cref{sec:exp}. We include further information regarding the data-generating process for the simulated experiment~\pcref{subsec:simulated_exp} and the ISTAnt dataset~\citep{cadei2024smoke} used in the ecological case study~\pcref{subsec:istant}, 
as well as additional experimental results.

\subsection{Numerical Simulation}
\label{app:numerical_simulation}

\textbf{Experiment setting.}
We consider five causal variables $(\Zb_1,\cdots,\Zb_5)$ generated based on a linear structural causal model~\citep{peters2017elements} 
\begin{equation*} 
\Zb = B \Zb + \varepsilon, 
\end{equation*}
where $\Zb \coloneqq (\Zb_1, \Zb_2, \Zb_3, \Zb_4, \Zb_5)$, $\Zb$ takes values in $\bbR^5$, 
$\varepsilon \sim\cN_5(0, I)$, and 
$B = \begin{bmatrix} 
0 & 0 & 0 & 0 & 0 \\ 
1 & 0 & 0 & 0 & 1 \\ 
1 & 1 & 0 & 0 & 0 \\
1 & 0 & 0 & 0 & 0 \\
1 & 0 & 0 & 1 & 0
\end{bmatrix}$,
which induces the partial DAG depicted in~\cref{fig:simu_measurement_model}. 
Two of the causal variables ($\Zb_4$ and $\Zb_5$) are observed (i.e., directly measured as in~\cref{def:measurement_model}), 
and the other three ($\Zb_1$, $\Zb_2$, and $\Zb_3$) are latent and we observe only a bijective 
mixing $\Xb$ of them. 

For the purpose of latent variable identification, we consider the multiview scenario in~\citep{yao2023multi} where two views $\Xb_1, \Xb_2$ are generated from different subsets of latent variables. Formally, we have
\begin{equation}
    \begin{aligned}
        \Xb_1 &= f_1(\Zb_1, \Zb_2)\\
        \Xb_2 &= f_2(\Zb_1, \Zb_3),
    \end{aligned}
\end{equation}
where $f_1, f_2: \RR^2 \to \RR^2$ are diffeomorphisms, implemented using invertible MLPs as suggested by~\citet{yao2023multi}.



\textbf{Implementation details.}
We employ the latent variable identification algorithm proposed by \citet{yao2023multi}, which guarantees that the shared latent variables among different views can be identified up to a diffeomorphism in the sense of~\cref{def:identif}.
Thus, by utilizing $\Xb_1, \Xb_2$, we can obtain a nonlinear bijective transformation of their shared latent variable $\Zb_1$. This allows us to construct a measurement model 
$\Mcal = \langle \Zb, \widehat\Zb_{A_1}, \{h_1\}\rangle$ (see~\cref{fig:simu_measurement_model}), where 
$\Zb = \{\Zb_1, \cdots,\Zb_5\}$ and 
$\widehat\Zb_{A_1} = h(\Zb_1)$ for some (unknown) smooth invertible map $h: \bbR\to\bbR$.  

We train three CRL models following the implementation settings in~\citep[Tab.~4]{yao2023multi}. 
\begin{itemize}
    \item \textbf{Model A}: a sufficiently trained model (trained for $50001$ steps) from which we expect the learned representation $\widehat{\Zb}_{A_1}^A$ (where by a slight abuse of notation, the superscript represents the model indicator) to \emph{exclusively measure} $\Zb_1$; \item \textbf{Model B}: an insufficiently trained model (trained for $51$ steps) with unclear latent-measurement correspondence;
    \item \textbf{Model C}: a corrupted version of Model A where the representation $\widehat{\Zb}_{A_1}^C := \widehat{\Zb}^A_{A_1} + 0.2\Zb_2 - 0.1\Zb_3$, i.e., a linear mixing of the representation $\widehat{\Zb}_{A_1}^A$ from Model A, and $\Zb_2, \Zb_3$.
\end{itemize}

For each of the three trained models, we generate
50 independent datasets, each containing 4096 paired samples of $\Zb, \widehat{\Zb}_{A_1}$.
We compute the respective T-MEX scores based on these generated datasets for all three models, using the the projected covariance measure \citep[PCM,][]{lundborg2024projected} implemented in \texttt{pycomets} \citep{huang_pycomets} using linear regression models to estimate the conditional means (see~\cref{app:comets}).

\begin{figure}[!hbt]
    \centering
    \includegraphics[width=.7\linewidth]{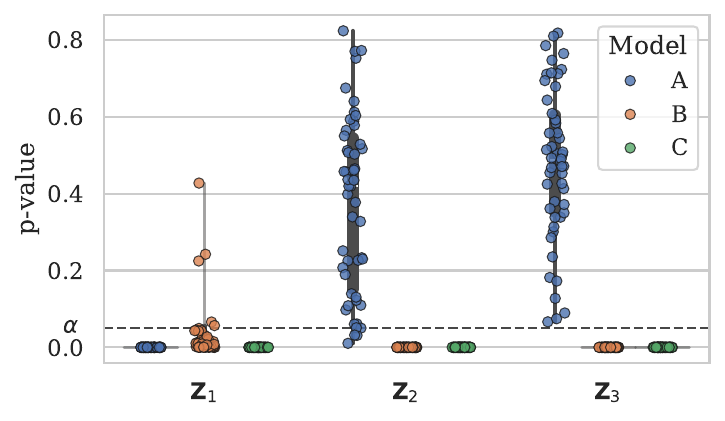}
    \caption{Violin plots of p-values from testing the conditional independencies $\widehat\Zb_{A_1}\indep \Zb_i \given \Zb_{[5]\setminus i}$ for $i\in[3]$ based on the PCM tests~\citep{lundborg2024projected}. The black dashed line is at the significance level $\alpha = 0.05$. A p-value $< \alpha$ for $\Zb_i$ means there is an edge from $\Zb_i$ to the measurement $\widehat{\Zb}_{A_1}$.}
    \label{fig:simu_pvals}
\end{figure}

\textbf{Additional results.} Since T-MEX relies on statistical testing, we further assess its statistical validity by examining the underlying p-values that lead to the test results and the T-MEX score.
\cref{fig:simu_pvals} shows the p-values resulted from testing each of the three null hypotheses:
\begin{equation*}
    \cH_0(i): \widehat\Zb_{A_1}\indep \Zb_i \given \Zb_{[5]\setminus i} \text{ for } i\in[3].
\end{equation*}
We omit $\Zb_4$ and $\Zb_5$ since they are not involved in generating the two views $\Xb_1$ and $\Xb_2$. 
 
\Cref{fig:simu_pvals} shows that Model A aligns with the measurement model in~\cref{fig:simu_measurement_model}, evidenced by (i) small p-values for $\cH_0(1)$ and (ii) approximately uniformly distributed p-values for both $\cH_0(2)$ and $\cH_0(3)$, given a valid test (see~\cref{app:comets} for further explanations).
In contrast, for Models B and C, nearly all p-values are smaller than $\alpha$, leading to rejections of the null hypotheses, which indicates that the learned representation $\widehat{\Zb}_{A_1}$ is a mixture of all three causal variables $\Zb_1, \Zb_2, \Zb_3$, and thus fails to exclusively measure $\Zb_1$.

\textbf{Computational resources.} We train the CRL models (model A, B, C) using a single node GPU (\texttt{NVIDIA GeForce RTX1080Ti}) with $10\mathrm{GB}$ of RAM, $4$ CPU cores for less than one GPU hour. ATE estimation and T-MEX computation take less than one minute on a standard CPU.






\subsection{Real-World Ecological Experiment: ISTAnt}
\label{app:istant}
\textbf{Experiment Setting.} ISTAnt is a real-world ecological benchmark designed to evaluate learned representations on downstream causal inference tasks from high-dimensional observational data. It comprises 44 ant-triplet video recordings collected through a randomized controlled trial. This benchmark adopts the problem formulation introduced by \citet{cadei2024smoke}, aiming to estimate the causal effect of specific treatments (e.g., chemical exposure) on ants’ social behavior, particularly grooming events. The experimental design and recording setup are shown in~\cref{fig:experiment_setup}; for further details, refer to~\citep[App. C]{cadei2024smoke}.

In ISTAnt, each observation (video recording) $i$ is associated with a treatment assignment $\Tb_i$ and a set of experimental covariates $\Wb_i$ (including experiment day, time of the day, batch, position within the batch, and annotator). However, only a subset of videos is annotated with the outcome of interest $\Yb_i$ (i.e., grooming events), which hinders reliable causal inference at a population level, such as treatment effect estimation. To address this challenge, ~\citet{cadei2024smoke} proposes to train a classifier on top of a pre-trained feature extractor ~\citep[e.g., DINOv2, ][]{oquab2023dinov2} using this limited set of annotated samples, to impute missing labels while still enabling valid causal inference at the population level; specifically, for estimating the Average Treatment Effect (ATE).


\begin{table}
    \centering
    \begin{minipage}[t]{\textwidth}
    \centering
        \caption{Hyperparameters for the real-world ecological experiment~\pcref{subsec:istant,app:istant}, giving rise to 2,400 model configurations in total. All other settings follow~\citep[App. C]{cadei2024smoke}.}
        \label{tab:params_istant}
        \begin{tabular}{ll}
        \toprule
         \rowcolor{Gray!20} \textbf{Hyperparameter}   & \textbf{Value(s)} \\
         \midrule
          Input Preprocessing   & \texttt{YES} / \texttt{NO}\\
          Number of Hidden Layers & 1, 2\\
          Batch Size & 64, 128, 256\\
          Adam: learning rate& 5e-2, 1e-2, 5e-3, 1e-3, 5e-4\\
          \multirow{2}{*}{Training objective} & Empirical Risk, Invariant Risk \citep{arjovsky2020invariantriskminimization},  \\
                           &  vREx \citep{krueger2021out}, Deconfounded Risk \citep{cadei2025causal}   \\
          \# Seeds & 0,1, ..., 9\\
          \bottomrule
        \end{tabular}
    \end{minipage}
\vspace{-10pt}
\end{table}

\textbf{Implementation details.} Following~\citep{cadei2024smoke}, we train 2,400 classification heads on top of DINOv2~\citep{oquab2023dinov2}, varying the architecture and training settings, and estimate the causal effect using all video samples together with the predicted labels $\widehat{\Yb}$s by AIPW estimator~\citep{robins1994estimation}. The hyperparameter configurations are summarized in~\cref{tab:params_istant}, with all other implementation details following~\citep[App. C]{cadei2024smoke}.


By contrasting with the measurement model depicted in~\cref{fig:istan_measurement_model}, we compute the T-MEX scores for all 2,400 models. 
Since we focus on models with more than 80\% prediction accuracy~\pcref{subsec:istant}, the null hypothesis $\widehat{\Yb} \indep \Yb \given \Tb$ is rejected in all cases, consistently indicating $\Yb \to \widehat{\Yb}$. Thus, we only focus on the following null hypothesis:
\begin{equation*}
    \Hcal_0: \widehat{\Yb} \indep \Tb \given \Yb,
\end{equation*}
where $\widehat{\Yb}$ denotes the predicted label and $\Yb$ the ground truth one. 
A misalignment with the measurement model in~\cref{fig:istan_measurement_model} leads to rejecting $\Hcal_0$, resulting T-MEX=1, whereas as a causally valid representation $\widehat{\Yb}$ that exclusively measures $\Yb$ gives rise to T-MEX=0.
We summarize all results in~\cref{fig:exp} and provide extended discussions in~\cref{subsec:istant}.


\textbf{Computational resources.} We run all the analyses in~\cref{subsec:istant} using $48\mathrm{GB}$ of RAM, $20$ CPU cores, and a single node GPU (\texttt{NVIDIA GeForce RTX2080Ti}) for 24 GPU hours.
Data preprocessing and feature extraction using DINOv2 account for the majority of the computational time, whereas classifier training, AIPW estimation, and the T-MEX test contribute negligibly by comparison.

\begin{figure}[h]{}  
        \hfill
        \begin{subfigure}[b]{0.48\textwidth}
            \centering   \includegraphics[height=5cm]
            {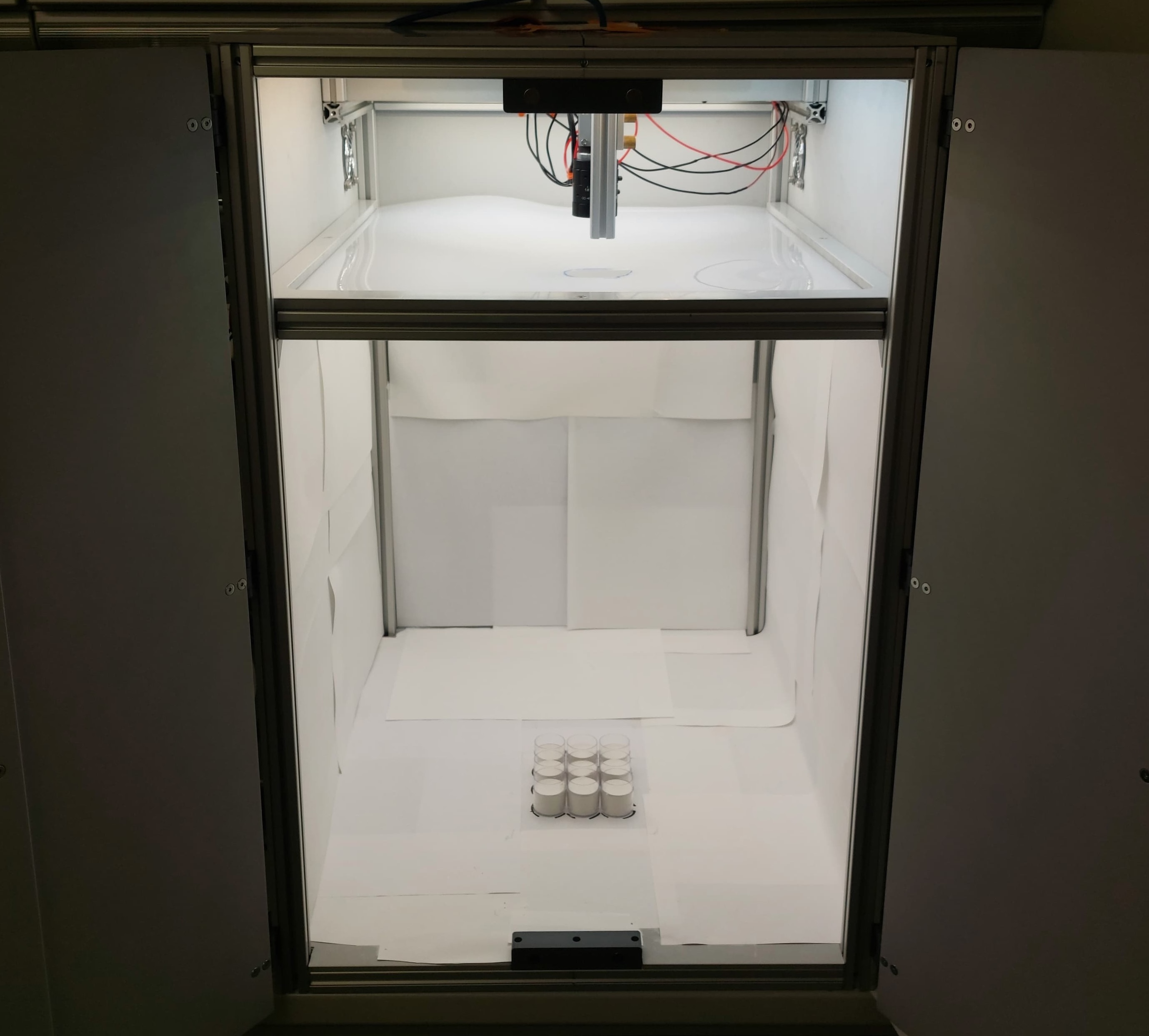} 
            \caption{Filming box}
            \label{fig:camera}
        \end{subfigure}
        \hfill
        \begin{subfigure}[b]{0.48\textwidth}
            \centering
            \includegraphics[height=5cm]{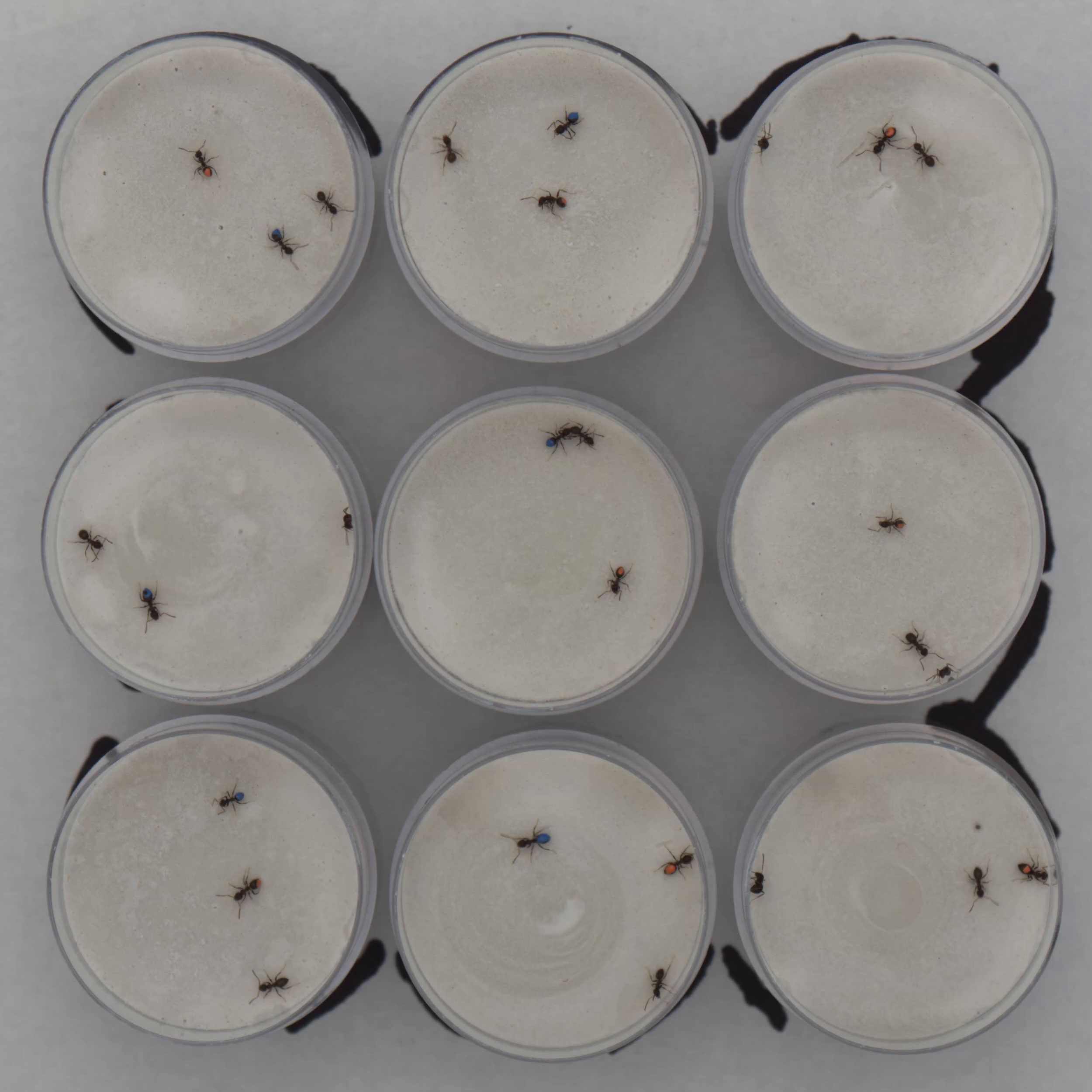}  
            \caption{Batch example}
            \label{fig:batch}
        \end{subfigure}
        \hfill
    \caption{Visualization of ISTAnt recording set-up \citep{cadei2024smoke}.}
    \label{fig:experiment_setup}
\end{figure}

\subsection{Caveats of Using SHD to Evaluate Causal Representations}
\label{subsec:false_positive_shd}
\begin{wrapfigure}{r}{0.4\textwidth}
    \centering
    \vspace{-18pt}
    \begin{tikzpicture}
    \tikzstyle{var}=[circle, draw, thick, minimum size=8.5mm, font=\small, inner sep=1]
    \tikzstyle{arrow}=[-latex, thick]
    \tikzstyle{doublearrow}=[latex-latex, thick]
    \tikzstyle{dashedarrow}=[-latex, thick, dashed]

    \node[var] (Z1) at (-1.5, 1.5) {$\Zb_1$};
    \node[var] (Z2) at (0, 1.5) {$\Zb_2$};
    \node[var] (Z3) at (1.5, 1.5) {$\Zb_3$};

    \node[var, fill=gray!30] (ZhA1) at (-1.5, 0) {$\widehat{\Zb}_{A_1}$};
    \node[var, fill=gray!30] (ZhA2) at (0, 0) {$\widehat{\Zb}_{A_2}$};
    \node[var, fill=gray!30] (ZhA3) at (1.5, 0) {$\widehat{\Zb}_{A_3}$};

    \draw[arrow] (Z1) -- (Z2); 
    \draw[arrow] (Z2) -- (Z3);
    \draw[-latex, thick] (Z1) to [out=45,in=145] (Z3);
    \draw[arrow] (Z1) -- (ZhA1);
    \draw[arrow] (Z2) -- (ZhA2);
    \draw[arrow] (Z2) -- (ZhA1);
    \draw[arrow] (Z3) -- (ZhA3);
\end{tikzpicture}
    \caption{Example measurement model, where $\widehat{\Zb}_{A_1}$ block-identifies $\Zb_1, \Zb_2$, $\widehat{\Zb}_{A_2}$ and $ \widehat{\Zb}_{A_3}$ identifies $\Zb_2, \Zb_3$ respectively.}
    \label{fig:measurment_model_shd}
    \vspace{-35pt}
\end{wrapfigure}
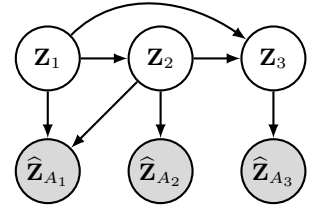

\textbf{Experiment Setting.} This experiment explores the potential pitfalls when directly using SHD to evaluate causal representations without properly evaluating the element-wise latent variable identification. Specifically, we consider a set of causal variables generated through the following structural equations:
\begin{equation}
\label{eq:exmaple_structure_eq}
    \begin{aligned}
        \Zb_1 &= \epsilon_1\\
        \Zb_2 &= \alpha_{12} \cdot \Zb_1 + \beta_2 \cdot \epsilon_2 \\
        \Zb_3 &= \alpha_{13} \cdot \Zb_1 + \alpha_{23} \cdot \Zb_2 + \beta_3 \cdot \epsilon_3, 
    \end{aligned}
\end{equation}
Assume the learned representation corresponds to the ground truth causal variable as follows:
\begin{equation}
\label{eq:example_measurement_eq}
    \begin{aligned}
        \widehat{\Zb}_{A_1} &= \gamma_{1} \cdot \Zb_1 + \gamma_{21} \cdot \Zb_2\\
        \widehat{\Zb}_{A_2} &= \gamma_2 \cdot \Zb_2\\
        \widehat{\Zb}_{A_3} &= \gamma_3 \cdot \Zb_3
    \end{aligned}
\end{equation}
where $\widehat{\Zb}_{A_1}$ remains a mixing of $\Zb_1$ and $\Zb_2$. The corresponding measurement model is shown in~\cref{fig:measurment_model_shd}.

\textbf{Implementation details.} We generate 100 different structure and measurement models following~\cref{eq:exmaple_structure_eq,eq:example_measurement_eq}, with all coefficients $\alpha$s and $\gamma$s sampled from $\text{Unif}[1, 10]$ and the $\beta$s sampled from $\text{Unif}[0.005, 0.02]$. We run LiNGAM~\citep{shimizu2006linear} from causal-learn~\citep{zheng2024causal} to discover the causal relationships between the measurements $\widehat{\Zb}_{A_1}, \widehat{\Zb}_{A_2}, \widehat{\Zb}_{A_3}$.

\textbf{Results.}~\cref{fig:false_positive_shd} shows the structural hamming distance of between the discovered graph on $\widehat{\Zb}$ and the ground truth one. Despite being entangled between $\Zb_1, \Zb_2$, $\widehat{\Zb}$ still yield the correct causal graph in most of the cases (77\%), as shown by the first bar in the plot. Hence, causal relations between the measurement variables should always be evaluated in conjunction with the variable identification. Otherwise, it can lead to misinterpretations as showcased by~\cref{fig:false_positive_shd}.

\textbf{Computational resources.} Data generating and causal discovery for~\cref{subsec:false_positive_shd} in total takes less than 10 minutes on a standard CPU.

\begin{figure}
    \centering
    \includegraphics[width=0.5\linewidth]{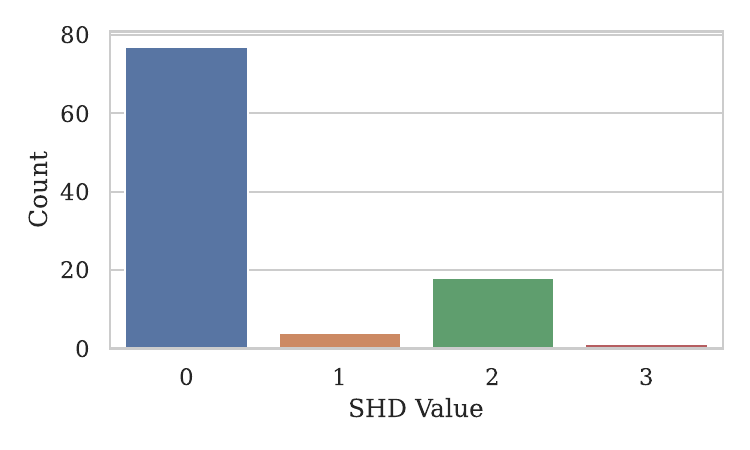}
    \caption{Structural Hamming Distance Values (SHD) of 100 structure and measurement models following~\cref{eq:exmaple_structure_eq,eq:example_measurement_eq}, where the measurement $\widehat{\Zb}_{A_1}$ is a mixing of the ground truth latent $\Zb_1, \Zb_2$. SHDs are computed between the discovered graph on $\widehat{\Zb}$ and the ground truth one.}
    \label{fig:false_positive_shd}
\end{figure}

\subsection{Additional Numerical Experiments Showcasing the Properties of T-MEX.} \label{app:add_exp}

In the following, \emph{T-MEX Oracle} is computed based on an oracle CI test with zero type I error and power 1. We compute the oracle T-MEX score to demonstrate desirable properties of the metric—such as its consistency and robustness across different conditional independence tests. This is feasible in the simulated setting, where the measurement functions are known by design. In contrast, for empirical representations learned by a CRL model, the true measurement functions are unknown, making oracle computation infeasible in practice.

\textbf{Scalability given higher-dimensional latent variables.} We simulate the causal variables based on a nonlinear SCM – a location-scale SCM as implemented by~\citet{liang2023causal}. Then, the measurement variables are simulated as a direct copy of each of the corresponding causal variables. For different numbers of latents (n-latent), we report the T-MEX score based on the Generalized Covariance Measure (GCM) test with linear regression (see \cref{app:comets} for more details about GCM), along with their standard error based on 20 random causal DAGs for the latent causal variables, each with $30$ repetitions and 1000 observations. In \cref{tab:high_dim_latent}, we see that T-MEX remains closely aligned with the T-MEX oracle in all cases, and can be efficiently computed up to 50 latents within a reasonable time, validating its applicability in moderate to high dimensions.

\begin{table}[htb]
\centering
\caption{T-MEX for higher-dimensional latents generated from nonlinear SCM}
\label{tab:T-MEX_higher_dims}
\begin{tabular}{lccc}
\toprule
\rowcolor{gray!20}\textbf{n-latent} & \textbf{T-MEX} & \textbf{T-MEX Oracle} & \textbf{time (sec)} \\
\midrule
5 & $0.0483 \pm 0.2223$ & 0 & $0.0359 \pm 0.0011$ \\
10 & $0.0033 \pm 0.0577$ & 0 & $0.1559 \pm 0.0017$ \\
20 & $0.0550 \pm 0.5187$ & 0 & $0.8080 \pm 0.0094$ \\
50 & $0.0917 \pm 0.9582$ & 0 & $10.8450 \pm 0.1004$ \\
\bottomrule
\end{tabular}
\label{tab:high_dim_latent}
\end{table}

\textbf{Consistent results when using different CI tests.} Following the same settings in \cref{subsec:simulated_exp}, we consider PCM test (see \cref{app:comets}) with random forest (RF) and Kernel Conditional Independence (KCI) test \citep{zhang2012kernel}. We see that T-MEX ranks the models consistently with the results in \cref{subsec:simulated_exp}. 

\begin{table}[h!]
\centering
\caption{T-MEX under different CI Tests}
\label{tab:T-MEX_CI_tests}
\begin{tabular}{lcc}
\toprule
\rowcolor{gray!20}\textbf{CI test} & \textbf{Model} & \textbf{T-MEX} \\
\midrule
PCM (RF) & A & $0.0000 \pm 0.0000$ \\
PCM (RF) & B & $0.8000 \pm 0.7559$ \\
PCM (RF) & C & $2.0000 \pm 0.0000$ \\
\addlinespace
KCI & A & $0.0400 \pm 0.1979$ \\
KCI & B & $0.1600 \pm 0.4218$ \\
KCI & C & $2.0000 \pm 0.0000$ \\
\bottomrule
\end{tabular}
\end{table}

\textbf{T-MEX under weak causal relations.}
Building on the experiment in~\cref{subsec:simulated_exp}, we next examine how T-MEX behaves under weak causal relations, and compare it with the previous CRL identifiability metrics such as $R^2$ and MCC. The data is generated from a linear SCM with three latent causal variables following the measurement model described in~\cref{fig:false_positive_shd}. The linear coefficients between the causal variables are sampled uniformly between 0.01 and 0.1.

\begin{table}[h!]
    \centering
    \caption{T-MEX under weak causal relations}
    \label{tab:T-MEX_weak_causal}
    \begin{tabular}{cccc}
    \toprule
        \rowcolor{gray!20} \textbf{T-MEX Oracle} & \textbf{T-MEX} & $\mathbf{R^2}$ & \textbf{MCC} \\
        \midrule
        1  & $1.0100 \pm 0.0995$ & $1.0000 \pm 0.0000$ & $1.0000 \pm 0.0000$\\
        \bottomrule
    \end{tabular}
\end{table}

Under identifiability assumptions that guarantee element-wise correspondence, T-MEX correctly detects the mixing effect and gives a score near one (note that T-MEX = 0 indicates perfect alignment).
In contrast, both $R^2$ and MCC fail to reflect this violation. Despite the entanglement in $\widehat{\Zb}_1$, they still assign the maximum score of 1, misleadingly suggesting perfect element-wise identification (see~\cref{tab:T-MEX_weak_causal}).

\textbf{Alignment between T-MEX and ATE under nonlinear SCM.}
Extending~\cref{subsec:simulated_exp}, we further examine the correspondence between T-MEX and ATE under the more general nonlinear setting. We consider the same causal graph as given in~\cref{subsec:simulated_exp}, where $\Zb_1$ confounds $\Zb_2$ and $\Zb_3$, and needs to be adjusted for a valid treatment effect estimation. A ``perfect" model means $\widehat{\Zb}_1$ exclusively measures $\Zb_1$, whereas a mixed model indicates an entangled representation, i.e., $\widehat{\Zb}_1$ mixes $\Zb_1$ and $\Zb_2$.

\begin{table}[h!]
    \centering
    \caption{T-MEX v.s. ATE bias under nonlinear SCM}
    \label{tab:t-mex_ate_nonlinear}
    \begin{tabular}{ccc}
    \toprule
       \rowcolor{gray!20}\textbf{Model}  & T-\textbf{MEX} & \textbf{abs(ATE bias)}  \\
       \midrule
       Perfect  & $0.0100 \pm 0.1000$ & $0.1532 \pm 0.0283$\\
       Mixed & $1.0000 \pm 0.0000$ & $0.5357 \pm 0.0540$\\
       \bottomrule
    \end{tabular}
\end{table}

As shown in~\cref{tab:t-mex_ate_nonlinear}, the results are highly similar to the linear case~\cref{subsec:simulated_exp}, T-MEX closely aligns with the absolute values of the ATE bias, effectively evaluating causal representations for downstream inference tasks with nonlinear causal relations.

\textbf{Consistency of T-MEX under noisy measurements.}
We compare the empirical T-MEX value with/without noise. The noisy measurements are generated by adding Gaussian noise to the original latents. I.e., the measurement function writes 
\begin{equation*}
    h(\zb) = \zb + e 
\end{equation*}
with $e$ independent Gaussian noise. Results are evaluated over 100 datasets with 1000 samples each.
~\cref{tab:T-MEX_noise} shows that the empirical T-MEX score largely remains consistent under noisy measurements, validating its practical usability.
 
\begin{table}[h!]
\centering
\caption{T-MEX w/wo noise}
\label{tab:T-MEX_noise}
\begin{tabular}{ccc}
\toprule
\rowcolor{gray!20} \textbf{T-MEX Oracle}  &  \textbf{without noise} & \textbf{with noise}\\
\midrule
 0 & $0.0100 \pm 0.0100$ & $0.3900 \pm 0.6497$\\
\bottomrule
\end{tabular}
\end{table}

\section{Background on Conditional Independence Testing} \label{app:comets}

Testing conditional independence of two random variables $\Xb$ and $\Yb$ given a third random variable $\Zb$ 
is known to be a difficult problem if $Z$ is a continuous variable \citep{shah2020hardness}. 
The goal of conditional independence test is to test the null hypothesis 
\begin{equation*}
    \cH_0: \Xb \indep \Yb \given \Zb.
\end{equation*}
\citet{shah2020hardness} have shown that there is no valid test (i.e., a test that guarantees a Type I error rate 
to be no larger than the given significance level $\alpha$) that has power against all alternatives. 

Consider univariate variables $\Xb, \Yb, \Zb$, the generalized covariance measure (GCM) test proposed in \citet{shah2020hardness} aims to test an 
implication of conditional independence which can be written as the following null hypothesis:
\begin{equation*}
    \cH_0^{\text{GCM}}: \bbE[(\Yb - \bbE[\Yb\given \Zb])(\Xb - \bbE[\Xb\given \Zb])] = 0. 
\end{equation*}
The validity of the GCM test thus relies on that the conditional means $\bbE[\Yb\given \Zb]$ and 
$\bbE[\Xb\given \Zb]$ can be learned at sufficiently fast rates. 
It turns out that GCM does not have power against any alternative for which 
$\bbE[\cov(\Xb, \Yb \given \Zb)] = 0$ but $\Xb \dep \Yb \given \Zb$~\citep{lundborg2024projected}.

The projected covariance measure (PCM) proposed by \citet{lundborg2024projected} improves the power issue of GCM by testing a different implication of conditional independence:
\begin{equation*}
   \cH_0^{\text{PCM}}: \bbE[\Yb\given \Xb, \Zb] = \bbE[\Yb \given \Zb]. 
\end{equation*}
Similar to GCM, to ensure its validity, PCM also requires that the conditional means can be learned sufficiently fast, which is satisfied in our experiments~\pcref{sec:exp}.

\looseness=-1 There are other conditional independence tests such as mutual information based methods \citep{ai2024testing, runge18conditional} and kernel-based methods \citep{fernandez2024general, strobl2019approximate, zhang2012kernel}. 
We opted for PCM in our experiments for its computational advantage and theoretical guarantees on 
its validity under a flexible, model-agnostic framework. More discussions on the usage of PCM 
and GCM can be found in \citet{kook2024algorithm}. 
Notably, T-MEX is a general evaluation metric for causal representations that does not specify any particular type of tests, allowing practitioners to choose other testing methods that are more suitable for their problem settings.

\section{Extended Discussion}
\label{app:causal_implication}
This section elaborates on the implications of learned representations for downstream causal tasks. As briefly discussed in the main paper (following~\cref{def:causally_valid_model}), a representation is \emph{causally valid} (\cref{def:causally_valid_model}) with respect to a statistical estimand 
if and only if the statistical estimand remains unchanged when plugging in the measurement variables 
correspond to the causal variables.
More concretely, we illustrate the implications of nonlinear invertible reparameterizations 
of causal variables in two commonly encountered scenarios: 
when representations serve as proxies of (i) the treatment or outcome variables, 
and (ii) the confounders or instrumental variables.

\begin{wrapfigure}{r}{0.4\textwidth}
    \centering
    \begin{tikzpicture}
    \tikzstyle{var}=[circle, draw, thick, minimum size=8.5mm, font=\small, inner sep=1]
    \tikzstyle{arrow}=[-latex, thick]
    \tikzstyle{doublearrow}=[latex-latex, thick]
    \tikzstyle{dashedarrow}=[-latex, thick, dashed]

    \node[var] (Z1) at (-1.5, 1.5) {$\Zb_1$};
    \node[var] (Z2) at (0, 1.5) {$\Zb_2$};
    \node[var] (Z3) at (1.5, 1.5) {$\Zb_3$};

    \node[var, fill=gray!30] (ZhA1) at (-1.5, 0) {$\widehat{\Zb}_{A_1}$};
    \node[var, fill=gray!30] (ZhA2) at (0, 0) {$\widehat{\Zb}_{A_2}$};

    \draw[arrow] (Z1) -- (Z2); 
    \draw[arrow] (Z2) -- (Z3);
    \draw[-latex, thick] (Z1) to [out=45,in=145] (Z3);
    \draw[arrow] (Z1) -- (ZhA1);
    \draw[arrow] (Z2) -- (ZhA2);
\end{tikzpicture}
    \caption{$\widehat{\Zb}_{A_i}$ measures $\Zb_i$ through a nonlinear bijection for both $i=1, 2$.}
    \label{fig:measure_treatment_effect}
\end{wrapfigure}

\subsection{Representations of Treatment and Outcome}
\label{app:model_treatement_and_outcome}
\looseness=-1 Assume in~\Cref{fig:measure_treatment_effect} that $\widehat{\Zb}_{A_1}, \widehat{\Zb}_{A_2}$ are element-wise nonlinear invertible reparametrization of $\Zb_1, \Zb_2$  respectively; i.e., 
$\forall i \in \{1, 2\}, \widehat{\Zb}_{A_i} = h_i(\Zb_i)$ for some diffeomorphism $h_i: \bbR \to \bbR$.
We aim to estimate the treatment effect of $\Zb_1 \to \Zb_2$ using the learned representations $\widehat{\Zb}_{A_1}$ and $\widehat{\Zb}_{A_2}$.

Assume the $\Zb_2$ is generated following~\cref{eq:struc_eq_cs}, i.e.,
\begin{equation*} 
    \Zb_2 \coloneqq a \cdot \Zb_1 + e
\end{equation*}
with $e\sim P_e$, $\bbE[e] = 0$ and $e\indep \Zb_1$. 
Given there is no unobserved confounding, the ground truth average treatment effect is written as
\begin{equation}
    \text{ATE}(\Zb_1 \to \Zb_2) = \dfrac{\partial \bbE[\Zb_2 ~|~ do(\Zb_1 = \zb_1)]}{\partial \zb_1} = \dfrac{\partial \bbE[\Zb_2 ~|~ \Zb_1 = \zb_1]}{\partial \zb_1} = \dfrac{\partial \bbE[a \zb_1 + e]}{\partial \zb_1}
        = a.
\end{equation}

We assume measurement function $h_i$ for all $i \in\{1, 2\}$ to be linear, i.e., 
\begin{equation}
    \begin{aligned}
        \widehat{\Zb}_{A_1} = \alpha_1 \cdot \Zb_1, \qquad
        \widehat{\Zb}_{A_2} = \alpha_2 \cdot \Zb_2, \quad \text{and} \quad 
        \alpha_1, \alpha_2 \neq 0.
    \end{aligned}
\end{equation}

The ATE estimand from the learned representations yields:
\begin{equation}
\label{eq:ate_from_z_hat}
\begin{aligned}
    \text{ATE}(\widehat{\Zb}_{A_1} \to \widehat{\Zb}_{A_2}) 
    &= \dfrac{\partial \bbE[\widehat{\Zb}_{A_2} ~|~ \widehat{\Zb}_{A_1} = \hat{\zb}_{A_1}]}{\partial \hat{\zb}_{A_1}}\\
    &= \dfrac{\partial \bbE[\alpha_2 \Zb_2~|~ \alpha_1 \Zb_1 = \alpha_1 \zb_1]}{\partial \alpha_1 \zb_1 } \\
    &= \dfrac{\alpha_2 \partial  \bbE[\Zb_2 ~|~  \Zb_1 = \zb_1]}{ \alpha_1 \partial \zb_1}
    = \frac{\alpha_2}{\alpha_1} a.
\end{aligned}
\end{equation}

As shown by~\cref{eq:ate_from_z_hat}, the ATE estimand using the learned representation $\widehat{\Zb}_{A_1}$ and $\widehat{\Zb}_{A_2}$ can be arbitrarily scaled by the factor of $\nicefrac{\alpha_2}{\alpha_1}$. 
Thus, measurements that bijectively transform the causal latent variables cannot naively support estimating the treatment effect, violating causal validity~\pcref{def:causally_valid_model}; it requires direct supervision or observation on \emph{both} treatment and outcome variables, as also pointed out by~\citep[Sec. 4]{von2024nonparametric}.

On the other hand, information-theoretic measures for quantifying causal influence remain invariant under bijective transformation, such as the mutual information $I_{\text{int}}(\Zb_1; \Zb_2) = I_{\text{int}}(\widehat{\Zb}_{A_1}; \widehat{\Zb}_{A_2})$, as shown by \citet{janzing2013quantifying}. 

\subsection{Representations of Confounders or Instruments}
\label{app:model_confounding_instrument}

\textbf{Measuring confounding.} We first show an example where an observed treatment $\Tb$ 
and an observed outcome $\Yb$ is confounded by a third variable $\Wb$ which is measured by $\widehat{\Wb} = h(\Wb)$ through a deterministic invertible function $h$.

Formally, the measurement model is defined as  $\Mcal^{\text{conf}} = \langle \Zb, \widehat\Zb, \{h\}\rangle$ with $\Zb = \{\Tb,\Yb, \Wb\}$ and $\widehat\Zb = \{\widehat{\Wb}\}$, where $\Tb, \Yb$ are \emph{directly measured}~\pcref{def:measurement_model}.
The corresponding DAG is given in~\cref{fig:supervised_TY}. 
We show in the following that this measurement model is indeed causally 
valid~\pcref{def:causally_valid_model} with respect to the statistical estimand for the 
Average Treatment Effect (ATE) of $\Tb$ on $\Yb$. 

\begin{figure}[t]
    \centering
    \begin{minipage}[t]{0.48\textwidth}
        \centering
        \begin{tikzpicture}
    \tikzstyle{var}=[circle, draw, thick, minimum size=8.5mm, font=\small, inner sep=1]
    \tikzstyle{arrow}=[-latex, thick]
    \tikzstyle{doublearrow}=[latex-latex, thick]
    \tikzstyle{dashedarrow}=[-latex, thick, dashed]

    \node[var] (W) at (-1.5, 1.5) {$\Wb$};

    \node[var, fill=gray!30] (What) at (-1.5, 0) {$\widehat{\Wb}$};
    \node[var, fill=gray!30] (T) at (0, 1.5) {$\Tb$};
    \node[var, fill=gray!30] (Y) at (1.5, 1.5) {$\Yb$};

    \draw[arrow] (W) -- (What); 
    \draw[arrow] (W) -- (T);
    \draw[-latex, thick] (W) to [out=45,in=145] (Y);
    \draw[arrow] (T) -- (Y);
\end{tikzpicture}
        \caption{\emph{ATE remains invariant under bijective transformation of confounders.} The treatment $\Tb$ and outcome $\Yb$ are directly measured (i.e., observed) whereas confounder $\Wb$ is measured by $\widehat{\Wb}$ through a nonlinear bijection.}
        \label{fig:supervised_TY}
    \end{minipage}
    \hfill
    \begin{minipage}[t]{0.48\textwidth}
        \centering
        \begin{tikzpicture}
    \tikzstyle{var}=[circle, draw, thick, minimum size=8.5mm, font=\small, inner sep=1]
    \tikzstyle{arrow}=[-latex, thick]
    \tikzstyle{doublearrow}=[latex-latex, thick]
    \tikzstyle{dashedarrow}=[-latex, thick, dashed]

    \node[var] (I) at (-1.5, 1.5) {$\Ib$};
    \node[var, dashed] (U) at (0.75, 3) {$\Ub$};

    \node[var, fill=gray!30] (Ihat) at (-1.5, 0) {$\widehat{\Ib}$};
    \node[var, fill=gray!30] (T) at (0, 1.5) {$\Tb$};
    \node[var, fill=gray!30] (Y) at (1.5, 1.5) {$\Yb$};

    \draw[arrow] (I) -- (Ihat); 
    \draw[arrow] (I) -- (T);
    \draw[arrow] (U) -- (T);
    \draw[arrow] (U) -- (Y);
    \draw[arrow] (T) -- (Y);
\end{tikzpicture}
        \caption{\emph{ATE remains invariant under bijective transformation of instruments.} $\widehat{\Ib}$ measures the instrument variable $\Ib$ through a nonlinear bijection. The treatment $\Tb$ and outcome $\Yb$ are directly measured (i.e., observed), and $\Ub$ denotes unobserved confounding.}
        \label{fig:model_instrument}
    \end{minipage}
\end{figure}


Under the standard assumptions for backdoor adjustment, it follows that 
\begin{equation}
\label{eq:distribution_invariant}
    \begin{aligned}
    \bbE(\Yb | do(\Tb=t)) &= \bbE_{\wb} \left[\bbE(\Yb ~|~\Wb, \Tb=t) \right]\\
    &= \int \bbE(\Yb ~|~\Wb, \Tb=t) P( \Wb) d\wb\\ 
    &= \int \bbE(\Yb ~|~ h^{-1}(\widehat{\Wb}), \Tb=t) P( h^{-1}(\widehat{\Wb})) \frac{d h^{-1}(\hat{\wb})}{d\hat{\wb}} d\hat{\wb} \\
    &= \int \bbE(\Yb ~|~ \widehat{\Wb}, \Tb=t) P(\widehat{\Wb}) d\hat{\wb}\\
    &= \bbE_{\hat\wb} \left[\bbE(\Yb ~|~\widehat\Wb, \Tb=t) \right],
\end{aligned}
\end{equation}

where we used the change of variable formula and the fact that $ \bbE(\Yb ~|~ \widehat{\Wb}, \Tb=t) =  \bbE(\Yb | h^{-1}(\widehat{\Wb}), \Tb=t)$. This is because $h^{-1}(\widehat{\Wb})$ is a sufficient statistic for $\Wb$~\citep[Ch. 6.2]{casella2024statistical} following $h$ is invertible.

Under the same assumptions, the ATE for \emph{binary} treatment 
can then be identified by the following statistical estimand
\begin{equation}
    \begin{aligned}
        \text{ATE}(\Tb \to \Yb) 
        &= \bbE [\Yb | do(\Tb=1)] - \bbE [\Yb ~|~ do(\Tb=0)] \\
        &= \bbE_{\wb} \left[\bbE(\Yb ~|~ \Wb, \Tb=1) - \bbE(\Yb ~|~ \Wb, \Tb=0)\right].
    \end{aligned}
\end{equation}
Following~\cref{eq:distribution_invariant}, we have
\begin{equation*}
      \text{ATE}(\Tb \to \Yb) =\bbE_{\hat\wb} \left[\bbE(\Yb ~|~ \widehat\Wb, \Tb=1) - \bbE(\Yb ~|~ \widehat\Wb, \Tb=0)\right],
\end{equation*}
indicating that the identified statistical estimand $\text{ATE}(\Tb \to \Yb)$ remains invariant for the measurement $\widehat{\Wb}$. Similarly, ATE also remains invariant when the treatment is continuous:
\begin{equation}
    \begin{aligned}
        \text{ATE}(\Tb \to \Yb) &= \dfrac{\partial \bbE [\Yb ~|~ do(\Tb=t)]}{dt} = \dfrac{\partial \bbE_{\wb}\bbE [\Yb ~|~ \Wb, \Tb=t]}{dt} = \dfrac{\partial \bbE_{\hat\wb}\bbE [\Yb ~|~ \widehat{\Wb}, \Tb=t]}{dt},
    \end{aligned}
\end{equation}
where the last equality holds because of~\cref{eq:distribution_invariant}. 
Therefore, we have shown that invertible reparameterizations of the confounders can be a drop-in replacement of the true confounding variables in the statistical estimand for ATE, 
for both discrete and continuous treatments, and thus this measurement model $\Mcal^{\text{conf}}$ 
is indeed causally valid for ATE.

\looseness=-1 \textbf{Measuring instrumental variables.} We now give a second example of 
ATE estimation under an instrumental variable setup. 
We assume that the instrument $\Ib$ is measured by $\widehat{\Ib} = h(\Ib)$ through a bijective transformation $h$.
We show that under certain assumptions, the statistical estimand does not change 
when using $\widehat\Ib$ as a drop-in replacement of the true instrument $\Ib$. 
We focus on the case where the instrument $\Ib$, the treatment $\Tb$, and the response $\Yb$ are 
all univariate continuous variables; further discussion on multivariate and discrete valued variables is beyond the scope of this paper. 
Formally, the measurement model is defined as $\Mcal^{\text{IV}} = \langle \Zb,\widehat{\Zb}, \{h\}\rangle$ with causal variables $\Zb = \{\Ib, \Tb, \Yb\}$ and measurement variables $\widehat\Zb = \{\widehat{\Ib}\}$. The treatment $\Tb$ and outcome $\Yb$ are \emph{directly measured}~\pcref{def:measurement_model}
and confounded by unknown hidden confounders $\Ub$.  
\Cref{fig:model_instrument} shows the DAG of this measurement model.

We show in the following that the instrument $\Ib$ remains a valid instrumental variable under a bijective transformation, i.e., the measurement variable $\widehat{\Ib} = h(\Ib)$ also satisfies the standard IV assumptions, which are listed as follows: 
\begin{itemize}
    \item Relevancy: $\Ib \dep\Tb \given \Ub$
    \item Unconfoundedness: $\Ib \indep\Ub$
    \item Exclusion restriction criteria: $\Ib \indep \Yb \given \Tb, \Ub$
\end{itemize}
Following standard probability theory~\citep[see e.g.,][]{billingsley2008probability}, if $h$ is a 
bijective function, all three conditions still hold when replacing $\Ib$ by $h(\Ib)$. 
This means that if the ATE is identified by a statistical estimand when using $\Ib$ as an instrument, 
it is also identified when using $\widehat{\Ib}$ as an instrument. 
In other words, the measurement model $\Mcal^{\text{IV}}$ is causally valid with respect to an identified statistical estimand because $\widehat{\Ib}$ can serve as a drop-in replacement for $\Ib$~\pcref{def:causally_valid_model}. 

As a specific example, consider the case where the causal mechanism of $\Yb$ is partially linear 
(a commonly studied setup in 
the semi-parametric inference literature, see e.g., \citet{chernozhukov2018double}), i.e.,
$\Yb = \Tb\beta + g(\Ub, \varepsilon)$, for some measurable function $g$ where $\bbE[g(\Ub, \epsilon)]=0$ and where
$\varepsilon\sim P_\varepsilon$ is an independent noise variable, 
the ATE 
\begin{equation*}
    \text{ATE}(\Tb \to \Yb) 
    = \dfrac{\partial \EE[\Yb ~|~ do (\Tb=\tb)]}{\partial \tb} 
    = \dfrac{\partial \EE[\tb\beta + g(\Ub, \varepsilon)]}{\partial \tb}
    = \beta
\end{equation*}
can be identified by the statistical estimand 
\begin{equation}
\label{eq:statistical_estimand_ATE}
    \text{ATE}(\Tb \to \Yb) = \displaystyle \frac{\cov(\Yb, \Ib)}{\cov(\Tb,\Ib)}.
\end{equation}

We show in the following that the statistical estimand $\text{ATE}(\Tb \to \Yb)$ in~\cref{eq:statistical_estimand_ATE} remains invariant 
when using $\widehat{\Ib}$ as a drop-in replacement for $\Ib$. 
Plugging in $\widehat{\Ib}$ in the numerator 
\begin{equation*}
\cov(\Yb,\widehat\Ib) = \bbE[\Yb\widehat\Ib] - \bbE[\Yb]\bbE[\widehat\Ib]
= \beta \left(\bbE[\Tb \widehat\Ib] - \bbE[\Tb]\bbE[\widehat\Ib]\right) = \beta \cov(\Tb,\widehat\Ib),
\end{equation*}
we have $\displaystyle \frac{\cov(\Yb,\widehat\Ib)}{\cov(\Tb,\widehat\Ib)} = \beta = \frac{\cov(\Yb, \Ib)}{\cov(\Tb,\Ib)}$. Therefore, we have shown another example where the measurement $\widehat{\Ib}$ can serve as a drop-in replacement for the latent instrumental variable $\Ib$ for downstream causal inference tasks.

\end{document}